\def\year{2020}\relax
\documentclass[letterpaper]{article} 
\usepackage{aaai20_isaim}  
\usepackage{times}  
\usepackage{helvet} 
\usepackage{courier}  
\usepackage[hyphens]{url}  
\usepackage{graphicx} 
\usepackage{graphicx}  
\usepackage{color}

\usepackage{enumerate}
\usepackage{amsmath,amsfonts,bm}

\def\eqref#1{equation~\ref{#1}}

\def\1{\bm{1}}

\def\vb{{\bm{b}}}

\def\vh{{\bm{h}}}

\def\vx{{\bm{x}}}
\def\vy{{\bm{y}}}
\def\vz{{\bm{z}}}

\def\evb{{b}}

\def\evg{{g}}
\def\evh{{h}}

\def\evx{{x}}
\def\evy{{y}}
\def\evz{{z}}

\def\mW{{\bm{W}}}

\DeclareMathAlphabet{\mathsfit}{\encodingdefault}{\sfdefault}{m}{sl}
\SetMathAlphabet{\mathsfit}{bold}{\encodingdefault}{\sfdefault}{bx}{n}

\def\sL{{\mathbb{L}}}

\def\sN{{\mathbb{N}}}

\def\sS{{\mathbb{S}}}

\def\sX{{\mathbb{X}}}

\usepackage{amsthm}

\newtheorem{lem}{Lemma}
\newtheorem{thm}[lem]{Theorem}

\newtheorem{prp}[lem]{Proposition}
\usepackage{graphicx}
\usepackage{algorithm}
\usepackage{algpseudocode}
\usepackage{pgfplots}
\usepackage{pgfplotstable}
\definecolor{color1}{HTML}{E41A1C}
\definecolor{color2}{HTML}{377DB8}
\definecolor{color3}{HTML}{4DAF4A}
\definecolor{color4}{HTML}{984EA3}
\definecolor{color5}{HTML}{FF7F00}
\definecolor{color6}{HTML}{A65628}
\definecolor{color2b}{HTML}{98BFE0}
\makeatletter
\newtheorem*{ref@theorem}{\ref@title}
\newcommand{\newreftheorem}[2]{%
\newenvironment{ref#1}[1]{%
 \def\ref@title{#2 \ref{##1}}%
 \begin{ref@theorem}}%
 {\end{ref@theorem}}}
\makeatother
\newreftheorem{thm}{Theorem}
\newreftheorem{lem}{Lemma}
\newreftheorem{prp}{Proposition}
\newreftheorem{cor}{Corollary}

\title{Empirical Bounds on Linear Regions of Deep Rectifier Networks}
\author{Thiago Serra,\textsuperscript{\rm 1} Srikumar Ramalingam\textsuperscript{\rm 2}\\ 
\textsuperscript{\rm 1}Bucknell University, USA\\ 
\textsuperscript{\rm 2}The University of Utah, USA\\
thiago.serra@bucknell.edu, srikumar@cs.utah.edu 
}

\pgfplotsset{compat=1.14}
\begin{document}

\maketitle

\begin{abstract}
We can compare the expressiveness of neural networks that use rectified linear units (ReLUs) by the number of linear regions, which reflect the number of pieces of the piecewise linear functions modeled by such networks. 
However, enumerating these regions is prohibitive  
and the known analytical bounds 
are identical for networks 
with 
same dimensions. 
In this work, we approximate the number of linear regions 
through empirical bounds based on features of the trained network and probabilistic inference.
Our first contribution is a method to sample the activation patterns defined by ReLUs using universal hash functions. 
This method is based on a Mixed-Integer Linear Programming~(MILP) formulation of the network 
and 
an algorithm for probabilistic lower bounds of MILP solution sets that we call MIPBound, 
which is 
considerably 
faster than exact counting and 
reaches values in 
similar orders of magnitude. 
Our second contribution is 
a tighter 
activation-based bound for the maximum number of linear regions, 
which is 
particularly  stronger in networks with narrow layers. 
Combined, 
these bounds yield a fast proxy for the number of linear regions 
of a deep neural network.
\end{abstract}

\section{Introduction}

Neural networks with piecewise linear activations have become increasingly more common along the past decade, 
in particular since~\cite{ReLUGood1,ReLUGood2}. 
The simplest and most commonly used among such forms of activation is the Rectifier Linear Unit~(ReLU), which outputs the maximum between 0 and its input argument~\cite{OriginReLU,CurrentDNN}. In the functions modeled by these networks with ReLUs, we can associate each part of the domain in which the network corresponds to an affine function with a particular set of units having positive outputs. We say that those are the active units for that part of the domain. Consequently, 
over its entire input domain, 
the network models a piecewise-linear function~\cite{Arora}. Counting these ``pieces'' into which the domain is split, which are often denoted as linear regions, 
is one way to compare the expressiveness of models defined by networks with different configurations or coefficients. The theoretical 
analysis of the number of input regions in deep learning dates back to at least~\cite{DeepArchitectures}, 
and 
%
recent experiments have shown 
that the number of regions 
relates to the accuracy of 
similar 
networks~\cite{BoundingCounting}.

The study of linear regions in different network configurations has led to some interesting observations. For example, 
in a rectifier network with~$n$~ReLUs, we 
learned that not all configurations -- and in some cases none -- can reach the ceiling of $2^n$ regions 
(the number of possible sets of active units). 
On the one hand, we can construct 
networks where the number of regions is exponential on network depth~\cite{Pascanu,Montufar14}. 
On the other hand, there is a bottleneck effect by which the number of active units on each layer affects how the regions are partitioned by subsequent layers  due to the dimension of the space containing the image of the function, up to the point that even shallow networks define more linear regions than their deeper counterparts as the input dimension approaches  $n$~\cite{BoundingCounting}. 
Due to the linear local behavior, the size and shape of linear regions have been explored for  
provable robustness.  In that case, one wants to identify large stable regions and move certain points away from their boundaries~\cite{WongLP,LargeMargin,Maximization,LocallyLinear}. 
To some extent, we may regard the number and the geometry of linear regions as complementary. 
We can also use linear regions of a network to obtain smaller networks that are equivalent or a global approximation~\cite{Transformations}. 
However, we need faster methods to count or reasonably approximate the number of linear regions to make such metric practical.

\subsection{Linear Regions and Network Expressiveness}

The literature on counting linear regions has mainly focused on bounding their maximum number. Lower bounds 
are obtained by constructing networks defining increasingly larger number of linear regions~\cite{Pascanu,Montufar14,Arora,BoundingCounting}. 
Upper bounds are proven using the theory of hyperplane arrangements~\cite{Zaslavsky}  
along with other analytical 
insights~\cite{Raghu,Montufar17,BoundingCounting}. So far, these bounds are only identical -- and thus tight -- in the case 
of one-dimensional inputs~\cite{BoundingCounting}. Both of these lines have explored deepening connections with 
polyhedral theory, but some of these results have also been recently revisited using tropical 
algebra~\cite{TropicalChicago,TropicalCornell}. 
The 
linear regions of a trained network 
can be enumerated  
as the projection of a Mixed-Integer Linear Program~(MILP) on the binary variables defining if each unit is active~\cite{BoundingCounting}.
Another recent line of work focuses on analytical results for the average number of linear regions in practice~\cite{Hanin2019a,Hanin2019b}.

Other methods to study neural network expressiveness include universal approximation theory~\cite{Cybenko1989}, 
VC dimension~\cite{Bartlett1998}, and trajectory length~\cite{Raghu}. A network architecture can be studied and analyzed based on the largest class of functions that it can approximate. 
For example, it has been shown that any continuous function can be modeled using a single hidden layer of sigmoid activation functions~\cite{Cybenko1989}. 
The popular ResNet architecture~\cite{He2016DeepRL} with a single ReLU in every hidden layer can be a universal approximator~\cite{Lin2018}. Furthermore, a rectifier network with a single hidden layer can be trained to global optimality in polynomial time on the data size, but exponential on the input dimension~\cite{Arora}. The use of trajectory length for expressiveness is related to linear regions, i.e., by changing the input along a one dimensional path we study the transition across linear regions. Certain critical network architectures using leaky ReLUs 
are identified to produce connected decision regions~\cite{Nguyen2018}. To avoid such degenerate cases, one needs to use sufficiently wide hidden layers. However, this result is mainly applicable to leaky ReLUs and not to standard ReLUs~\cite{Beise2018}.

\subsection{Counting and Probabilistic Inference}\label{sec:sat_lb}


Approximating the size of a set of binary vectors, such as those units that are active on each linear region of a neural network, has been extensively studied in the context of propositional satisfiability~(SAT). 
A SAT formula on a set $V$ of Boolean variables has to satisfy a set of predicates. 
Counting solutions of SAT formulas is \#P-complete~\cite{Toda},  
but one can 
approximate the number of solutions 
by 
making a relatively small number of solver calls to 
restricted formulas. 
%
%
%
This line of work relies on 
hash functions with good statistical properties 
to partition the set of solutions $S$ into subsets having approximately half of the solutions each. 
After restricting a given formula $r$ times to 
either subset, 
one can test if the subset is empty.  
Intuitively, 
$|S| \geq 2^r$ with some probability if these subsets are more often nonempty,  
or else $|S| < 2^r$. 
SAT formulas 
are often restricted 
with predicates that encode XOR constraints, 
which can be interpreted in terms of $0$--$1$ variables 
as restricting the sum of a subset $U \subset V$ of the variables to be either even or odd.  
%
%
%
%
%
XOR constraints are universal hash functions~\cite{UniversalHashing},  
which 
enable 
approximate counting in polynomial time with an NP-oracle~\cite{Sipser,Stockmeyer}. 
Interestingly, 
formulas with a unique solution are as hard as those with multiple solutions~\cite{VV}. 
From a theoretical standpoint, 
such approximations are thus not much harder than obtaining a feasible solution.

In the seminal MBound algorithm~\cite{MBound}, 
XOR constraints on sets of variables with a fixed size $k$ 
yield the probability that $2^r$ is either a lower or an upper bound. 
The probabilistic lower bounds are always valid but get better as $k$ increases, 
whereas the probabilistic upper bounds are only valid if $k = |V|/2$. 
In practice, these lower bounds can be good for small 
$k$~\cite{ShortXORs}. 
These ideas were later extended to constraint satisfaction problems~\cite{CountingCSP}. 
%
%
%
Some of the subsequent work has been influenced by uniform sampling results from~\cite{NUSampling},  
where the fixed size $k$ is replaced with an independent probability $p$ of including each variable in each  XOR constraint. 
That includes the ApproxMC and the WISH algorithms~\cite{ApproxMC,ErmonWish1}, 
which rely on finding more solutions of the restricted formulas 
but generate $(\sigma, \epsilon)$ certificates 
by which, 
with probability $1 - \sigma$, the result is within $(1 \pm \epsilon)|S|$. 
Later work has provided 
upper bound guarantees when $p < 1/2$, 
showing that the size of those sets can be $\Theta\big(log (|V|)\big)$~\cite{ErmonWish2,Zhao}. 
Others have tackled this issue differently. 
One approach limited 
the counting to any set of variables $I$ for which any assignment leads to at most one solution in $V$, 
denoting those as minimal independent supports~\cite{MinimalIndependentSupport,MinimalIndependentSupportExp}. 
Another approach broke 
with the independent probability $p$ 
by using each variable the same number of times across $r$ XOR constraints~\cite{ErrorCorrecting,FastFlexible}. 
Related work on MILP has only focused on upper bounds based on relaxations~\cite{UpperMIP}. 

\subsection{Contributions of This Paper}

We propose empirical bounds 
based on the weight and bias coefficients of trained networks, 
which are the first able to compare networks with same configuration of layers. 
We also 
suggest replacing 
the potential number of linear regions $N$ of an architecture with 
the \emph{Minimum Activation Pattern Size}~(MAPS) $\eta = \log_2 N$. 
This value can be interpreted as the 
number of units that any network should have in order to define as many linear regions as another network when adjacent linear regions map to distinct affine functions.

Our main technical contributions are the following:

\begin{enumerate}[(i)]

\item We 
introduce a probabilistic lower bound based on 
sampling 
the activation patterns of~the~trained network. 
More generally, we can approximate 
solutions of MILP formulations more efficiently than if directly extending SAT-based methods with the  \texttt{MIPBound} algorithm introduced in Section~\ref{sec:lb}. 
See results in Figure~\ref{fig:bounds}. 

\item We refine the best known upper bound by further exploring how units partition the input space. With the theory in Section~\ref{sec:ub}, we find that unit activity further contributes to the bottleneck effect caused by narrow layers (those with few units). 
See results in Table~\ref{tab:ub_gap}.

\end{enumerate}

\section{Preliminaries and Notations}

We consider feedforward Deep Neural Networks~(DNNs) with ReLU activations. Each network has $n_0$ input variables given by $\vx = [\evx_1 ~ \evx_2 ~ \dots ~ \evx_{n_0}]^T$ with a bounded domain $\sX$ and $m$ output variables given by $\vy = [\evy_1 ~ \evy_2 ~ \dots ~ \evy_m]^T$. Each hidden layer $l \in \sL = \{1,2,\dots,L\}$ has $n_l$ hidden neurons indexed by $i \in \sN_l = \{1, 2, \ldots, n_l\}$ with  outputs given by $\vh^l = [\evh_1^l ~ \evh_2^l \dots \evh_{n_l}^l]^T$. 
For notation simplicity, we may use $\vh^0$ for $\vx$ and $\vh^{L+1}$ for $\vy$. Let $\mW^l$ be the $n_l \times n_{l-1}$ matrix where each row corresponds to the weights of a neuron of layer $l$. 
Let $\vb^l$ be the bias vector used to obtain the activation functions of neurons in layer $l$. The output of unit $i$ in layer $l$ consists of an affine transformation $\evg_i^l = \mW_{i}^l \vh^{l-1} + \vb_i^l$ to which we apply the ReLU activation $\evh_i^l = \max\{0, \evg_i^l\}$. 

We regard the DNN as a piecewise linear function $F:\mathbb{R}^{n_0}\rightarrow \mathbb{R}^{m}$ that maps the input $\vx \in \sX \subset \mathbb{R}^{n_0}$ to $\vy \in \mathbb{R}^m$. 
Hence, the domain is partitioned into regions within which $F$ corresponds to an affine function, which we denote as linear regions. 
Following the 
literature convention~\cite{Raghu,Montufar17,BoundingCounting}, 
we characterize each linear region by the set of units that are active 
in that domain. 
For each layer $l$, let $\sS^l \subseteq \{1,\ldots, n_l\}$ be the activation set in which  $i \in S^l$ if and only if $\evh^l_i > 0$. 
Let $\mathcal{S} = (\sS^1, \ldots, \sS^{l})$ be the activation pattern aggregating those activation sets. 
Consequently, the number of linear regions defined by the DNN 
is 
the number of nonempty sets in $\vx$ among all possible activation patterns.

\section{Counting and MILP Formulations}\label{sec:mip}

We can represent each linear region defined by a rectifier network with $n$ hidden units on domain $\sX$ by a distinct vector in $\{0,1\}^n$, 
where each element denotes if 
a 
unit is active or not. 
Such 
vector can be embedded into an MILP formulation mapping 
network 
inputs to outputs~\cite{BoundingCounting}. 
For a neuron $i$ in layer $l$, 
we use 
such binary variable $\evz_i$, 
vector $\vh^{l-1}$ of inputs coming from layer $l-1$, 
variable $\evg^l_i$ for the value of the affine transformation $\mW_i^l \vh^{l-1} + \vb_i^l$, 
variable $\evh^l_i = \max\{0, \evg_i^l\}$ denoting the output of the unit, 
and a variable $\bar{\evh}^l_i$ denoting the output of a complementary fictitious unit $\bar{\evh}^l_i = max\left\{0, -\evg_i^l \right\}$: 
\begin{align}
    \mW_i^l \vh^{l-1} + \evb_i^l = \evg_i^l \label{eq:mip_unit_begin} \\
    \evg_i^l = \evh_i^l - \bar{\evh}_i^l  \label{eq:mip_after_Wb_begin} \\
    \evh_i^l \leq H_i^l \evz_i^l, \qquad 
    \bar{\evh}_i^l \leq \bar{H}_i^l (1-\evz_i^l) \\
    \evh_i^l \geq 0, \qquad 
    \bar{\evh}_i^l \geq 0, \qquad 
    \evz_i^l \in \{0, 1\} \label{eq:mip_unit_end}
\end{align}

\begin{figure}[b]
\centering
\includegraphics[scale=0.17]{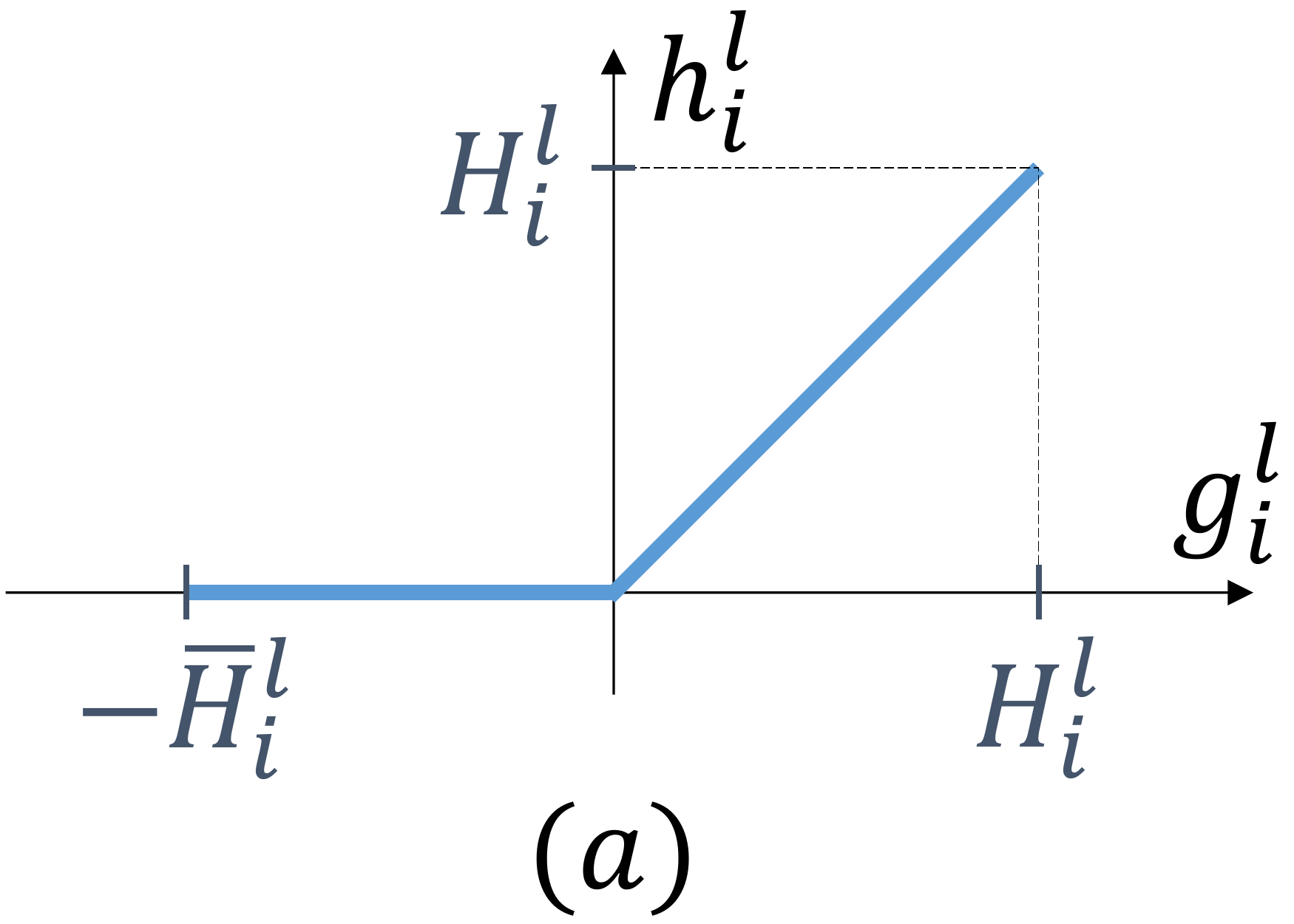}~
\includegraphics[scale=0.17]{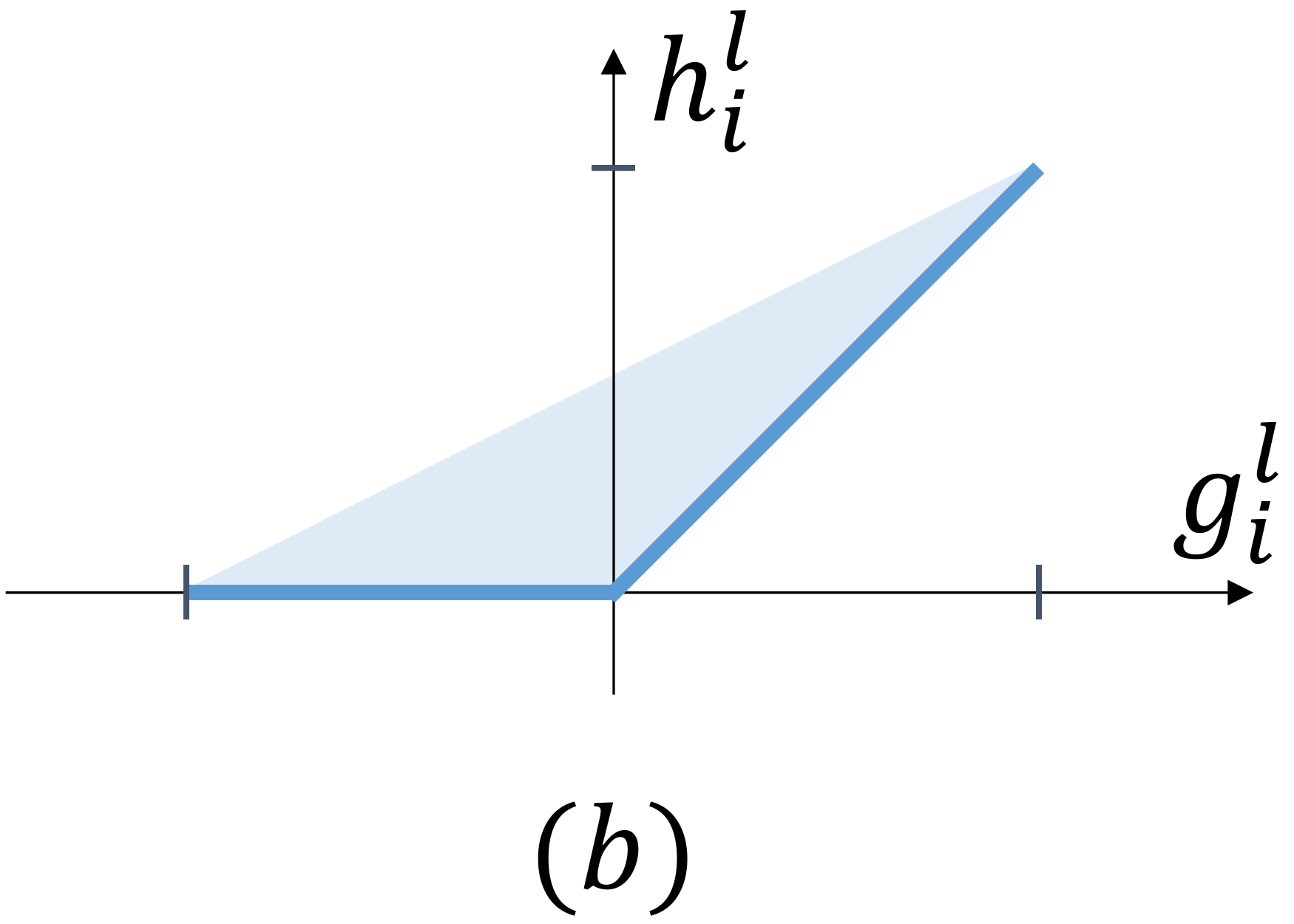}
\caption{(a) ReLU mapping $h_i^l = \max \{0, g_i^l\}$; 
and (b) Convex outer approximation on $(g_i^l, h_i^l)$.}
\label{fig:outer}
\end{figure}

For correctness, constants $H_i^l$ and $\bar{H}_i^l$ should be positive and as large as $\evh_i^l$ and $\bar{\evh}_i^l$ can be. 
In such case, 
the value of $\evg_i^l$ determines if the unit or its fictitious complement is active. 
However, 
constraints (\ref{eq:mip_unit_begin})--(\ref{eq:mip_unit_end}) allow $\evz_i^l = 1$ when $\evg_i^l = 0$. 
To count the number of linear regions, 
we 
consider the set of binary variables in the solutions where all active units have positive outputs, 
i.e., $\evh_i^l > 0$ if $\evz_i^l = 1$~\cite{BoundingCounting}, 
thereby counting the positive solutions with respect to $f$ on the binary variables of 
\begin{align}
    \max ~ & f \\
    \text{s.t.} ~ & (\ref{eq:mip_unit_begin})-(\ref{eq:mip_unit_end}), f \leq \evh_i^l + (1 - \evz_i^l) H_i^l & l \in \sL; ~ i \in \sN_l \\
    & \vx \in \sX
\end{align}
The solutions on $\evz$ can be enumerated using the one-tree algorithm~\cite{OneTree}, 
in which the branch-and-bound tree used to obtain the optimal solution of the formulation above is further expanded to collect near-optimal solutions up to a given 
limit. 
In general,
finding a feasible solution to a MILP is NP-complete~\cite{Cook} 
and thus optimization is NP-hard. 
However, 
a feasible solution 
in our case 
can be obtained from any valid input~\cite{FischettiMIP}. 
While that does not directly imply that optimization over neural networks is easy, 
it hints at 
good properties 
to explore. 

MILP formulations 
have been used 
for 
network verification 
~\cite{LomuscioMIP,DuttaMIP} and evaluation of adversarial perturbations
~\cite{ChengMIP,FischettiMIP,TjengMIP,XiaoMIP,StrongTrained}. 
Other applications relax 
the binary variables as continuous variables in 
$[0,1]$ or 
use 
the Linear Programming~(LP) formulation of a particular linear region~\cite{BastaniLP,EhlersLP,WongLP}, 
which 
is 
defined using $\mW_i^l \vh^{l-1} + \evb_i^l \geq 0$ for active units and the complement for inactive units.
Although equivalent, 
these MILP formulations may differ in strength~\cite{FischettiMIP,TjengMIP,HuchetteMIP,StrongTrained}.  
When the binary variables are relaxed, 
their 
linear relaxation may be different. 
We say that an MILP formulation $A$ is stronger than a formulation $B$ if, 
when projected on common sets of variables, 
the linear relaxation of $A$ is a 
subset of the linear relaxation of $B$. 
Formulation strength is often 
used 
as a proxy for 
solver performance.


%
%
%
%
Differences in strength may be due to 
smaller values for constants such as $H_i^l$ and $\bar{H}_i^l$,  
additional valid inequalities to remove fractional values of $\vz$, 
or an extended formulation with more variables. 
%
%
%
Let us 
consider the strength of the formulation for each ReLU activation $h_i^l = \max\{0,g_i^l\}$. 
Ideally, we want the projection on $g_i^l$ and $h_i^l$ to be the convex outer approximation of all possible combined values of those variables~\cite{WongLP}, 
as illustrated in Figure~\ref{fig:outer}. 

\begin{lem}\label{lem:outer}
If $H_i^l = \arg \max_{g^{l-1}} \{ g_i^l \} \geq 0$ and $\bar{H}_i^l = \arg \max_{g^{l-1}} \{ - g_i^l \} \geq 0$, then the linear relaxation of (\ref{eq:mip_after_Wb_begin})--(\ref{eq:mip_unit_end}) defines the convex outer approximation on $(g_i^l, h_i^l)$.
\end{lem}

Lemma~\ref{lem:outer} 
shows that 
the smallest possible values of $H_i^l$ and $\bar{H}_i^l$ 
are necessary to obtain a 
stronger 
formulation. The proof can be found in Appendix~\ref{app:outer}. 
A similar claim without proof is found in~\cite{HuchetteMIP}. 


When 
$\sX$ 
is defined by a box, 
and thus 
the domain of each input variable $x_i$ is an independent continuous interval, 
then the smallest possible values for $H_i^1$ and $\bar{H}_i^1$ can be computed with interval arithmetic 
by taking element-wise maxima~\cite{ChengMIP,BoundingCounting}. 
When extended to subsequent layers, 
this approach may overestimate 
$H_i^l$ and $\bar{H}_i^l$ 
as the 
maximum value of the outputs are not necessarily independent. 
More generally, 
if $\sX$ is polyhedral, 
we can obtain the smallest values for these constants by solving a sequence of MILPs~\cite{FischettiMIP,TjengMIP}. 
We can use $H_i^{l'} = \max ~  \{ g_i^{l'} : (\ref{eq:mip_unit_begin})-(\ref{eq:mip_unit_end}) ~\forall l \in \{1, \ldots, l'-1\}, ~ i \in \sN_l, x \in \sX \}$. 
If $H_i^{l'} \leq 0$, 
the unit is always inactive, denoted as \emph{stably inactive},  and we can remove such units from the formulation. 
Similarly, we can use 
$\bar{H}_i^{l'} = - \min ~  \{ g_i^{l'} : (\ref{eq:mip_unit_begin})-(\ref{eq:mip_unit_end}) ~\forall l \in \{1, \ldots, l'-1\}, ~ i \in \sN_l, x \in \sX \}$. 
If $\bar{H}_i^{l'} < 0$, 
the unit is always active, denoted as \emph{stably active}, and we can simply replace constraints (\ref{eq:mip_unit_begin})-(\ref{eq:mip_unit_end}) with $h_i^l = g_i^l$. 
%
In certain large networks, 
many units are stable~\cite{TjengMIP}. 
The remaining units, where activity depends on the input, are denoted \emph{unstable}.

We propose some valid inequalities involving consecutive layers of the network.  
%
For unit $i$ to be active when $b_i^l \leq 0$, 
there must be a positive contribution, 
and thus some unit $j$ in layer $l-1$ such that $\mW_{i j}^l > 0$ 
is 
also active. 
Hence, for each layer $l \in \{2, \ldots, L\}$ and unit $i \in \sN_l$ such that $b_i^l \leq 0$, 
\begin{equation}
    z_i^l \leq \sum_{j \in \{1, \ldots, n_{l-1}\} : \mW_{i j}^l > 0} z_j^{l-1} . 
    \label{eq:to_active}
\end{equation}
%
%
%
Similarly, unit $i$ is only inactive when $b_i^l > 0$ if 
some unit $j$ in layer $l-1$ such that $\mW_{i j}^l < 0$ is active. 
Likewise, for each layer $l \in \{2, \ldots, L\}$ and unit $i \in \sN_l$ such that $b_i^l > 0$,

\begin{equation}
    (1-z_i^l) \leq \sum_{j \in \{1, \ldots, n_{l-1}\} : \mW_{i j}^l < 0} z_j^{l-1} . 
    \label{eq:to_inactive}
\end{equation}

%
%
Let us denote unstable units in which $b_i^l \leq 0$, and thus (\ref{eq:to_active}) applies, as 
\emph{inactive leaning}; 
and those in which $b_i^l > 0$, and thus (\ref{eq:to_inactive}) applies, as 
\emph{active leaning}. 
Within linear regions where none among the units of the previous layer in the corresponding inequalities is active, 
these units can be regarded as stably inactive and stably active, respectively. 
We will use the same ideas to obtain better bounds in Section~\ref{sec:ub}.

\section{Approximate Lower Bound}\label{sec:lb}



%
%
%

We can use approximate model counting to estimate the size of the set of binary vectors corresponding to the activation patterns of a neural network. Essentially, we restrict the set of solutions by iteratively adding constraints with good sampling properties, such as XOR, until the problem becomes infeasible. Based on how many constraints it takes to make the formulation infeasible across many runs, we obtain bounds on the number of solutions with a certain probability. 
We describe in Section~\ref{sec:sat_lb} 
how this type of approach has been extensively studied in the SAT literature. 

The same ideas have not yet been extended to MILP formulations, where we exploit the specificity of MILP solvers to devise a more efficient algorithm. 
The assumption in SAT-based approaches is that each restricted formula entails a new call to the solver. 
Hence, obtaining a data point for each number of restrictions takes a linear number of calls. 
That has been 
improved to a logarithmic number of calls 
by orderly applying the same sequence of constraints up to each number of $r$ of XOR constraints, 
with which one can 
apply binary search to find the smallest $r$ that makes the formula unsatisfiable~\cite{LogSAT}.  
In MILP solvers, we can test for all values of $r$ with a single call to the solver by generating parity constraints as lazy cuts, which can be implemented through callbacks. When a new solution is found, a callback is invoked to generate parity constraints. Each constraint may or may not remove the solution just found, since we preserve the independence between the solutions found and the constraints generated, and thus we may need to generate multiple parity constraints before yielding the control back to the solver. 
Algorithm~\ref{alg:mip_bound} 
does that based on MBound~\cite{MBound}. 

\begin{algorithm}[h!]
\caption{\texttt{MIPBound} computes the probability of some lower bounds 
on the  solutions of a formulation $F$ with $n$ binary variables  
by adding parity constraints of size $k$}
\label{alg:mip_bound}
\footnotesize
\begin{algorithmic}[1]
\State $i \gets 0$ 
\For{$j \gets 0 \to n$}
\State $f[j] \gets 0$
\EndFor
\While{Termination criterion not satisfied}
\State $F' \gets F$ 
\State $i \gets i+1$ 
\State $r \gets 0$ 
\While{$F'$ has some solution $s$}
\Repeat 
\State Generate parity constraint $C$ with $k$ of the variables 
\State $F' \gets F' \cap C$
\State $r \gets r + 1$
\Until{$C$ removes $s$} 
\EndWhile
\For{$j \gets 0 \to r-1$}
\State $f[j] \gets f[j]+1$ 
\EndFor
\EndWhile
\For{$j \gets 0 \to n-1$} 
\State $\delta \gets f[j+1]/i - 1/2$
\If{$\delta > 0$} 
\State $P_j \gets 1 - \left(\frac{e^{2.\delta}}{(1+2.\delta)^{1+2.\delta}}\right)^{i/2}$ 
\Else 
\State \textbf{break} 
\EndIf
\EndFor
\State \textbf{return} Probabilities $P$
\end{algorithmic}
\end{algorithm}

We denote Algorithm~\ref{alg:mip_bound} as \texttt{MIPBound}. For each repetition of the outer \texttt{while} loop, 
parity constraints are added to a copy $F'$ of the formulation until it becomes infeasible. 
Appendix~\ref{ap:xor} discusses how to represent parity constraints in MILP using unit hypercube cuts from~\cite{CanonicalCuts}.
The inner \texttt{while} loop corresponds to the solver call 
and the block between \texttt{repeat} and \texttt{until} is implemented as a lazy cut callback, which is invoked when an incumbent solution $s$ is found. 
After each solver call, 
the number of constraints $r$ to make $F'$ infeasible is used to increase $f[j]$ for all $j < r$, 
which counts the number of times that $F'$ remained feasible after adding $j$ constraints. 
If $F$ often remained feasible 
with $j$ constraints, 
we compute the probability $P_{j-1}$ that $|S| > 2^{j-1}$, which is explained in 
Appendix~\ref{ap:probs}.
%


\section{Analytical Upper Bound}\label{sec:ub}

%
%
%




In order to bound the number of linear regions, 
we use 
activation hyperplanes and the theory of hyperplane arrangements. 
For each unit, 
the activation hyperplane $\mW_i^l \vh^{l-1} + b_i^l = 0$ splits the input space $\vh^{l-1}$ into the regions where 
the unit is active ($\mW_i^l \vh^{l-1} + b_i^l > 0$) or inactive ($\mW_i^l \vh^{l-1} + b_i^l \leq 0$). 
The number of full-dimensional regions defined by the arrangement of $n_l$ hyperplanes in an $n_{l-1}$-dimensional space is at most $\sum_{j=0}^{n_{l-1}} \binom{n_l}{j}$~\cite{Zaslavsky}.
See Appendix~\ref{app:bounds} for related bounds~\cite{Raghu,Montufar17,BoundingCounting}.



We can improve on these previous bounds by leveraging that some units of the trained network are stable for some or all possible inputs.
First, 
note that only units in layer $l$ that can be active in a given linear region produced by layers $1$ to $l-1$ affect the dimension of the space in which the linear region can be further partitioned by layers $l+1$ to $L$. 
Second, 
only the subset of these units 
that can also be inactive within that region, i.e., the unstable ones, 
counts toward the number of hyperplanes partitioning the linear region at layer $l$. 
Every linear region is contained in the same side of the hyperplane defined by each  stable unit. 
Hence, 
let $\mathcal{A}_l(k)$ be the maximum number of units that can be active in layer $l$ if $k$ units are active in layer $l-1$;  
and $\mathcal{I}_l(k)$ be the corresponding maximum number of units that are unstable.  

\begin{thm}\label{thm:upper_bound}
Consider a deep rectifier network with $L$ layers with input dimension $n_0$ and  
at most $\mathcal{A}_l(k)$ active units and $\mathcal{I}_{l}(k)$ unstable units in layer $l$ for
every linear region defined by layers $1$ to $l-1$ when $k$ units are active in layer $l-1$. 
Then the maximum number of linear regions is at most
\begin{align*}
\sum_{(j_1,\ldots,j_L) \in J} \prod_{l=1}^L \binom{\mathcal{I}_l(k_{l-1})}{j_l}
\end{align*}
where $J = \{(j_1, \ldots, j_L) \in \mathbb{Z}^L: 0 \leq j_l \leq \zeta_l$, $\zeta_l = \min\{n_0, k_1, \ldots, k_{l-1}, \mathcal{I}_l(k_{l-1})\}\ \}$ 
with $k_0 = n_0$ and 
$k_l = \mathcal{A}_{l}(k_{l-1}) - j_{l-1}$ for $l \in \sL$. 
\end{thm}
\begin{proof}
We define 
a recursive recurrence to 
bound the number of subregions within a region. 
Let $R(l,k,d)$ 
bound the maximum 
number of regions attainable from partitioning a region with dimension at most $d$ among those defined by layers $1$ to $l-1$ in which at most $k$ units are active in layer $l-1$. 
For the base case $l=L$, 
we have 
$R(L,k,d) = \sum_{j=0}^{\min\{\mathcal{I}_L(k),d\}} \binom{\mathcal{I}_L(k)}{j}$ since $\mathcal{I}_l(k) \leq \mathcal{A}_l(k)$. 
The recurrence groups regions with same number of active units in layer $l$ as  
$R(l,k,d) = \sum_{j=0}^{\mathcal{A}_l(k)} N^l_{\mathcal{I}_l(k), d, j} R(l+1, j, \min\{j, d\})$ for $l = 1$ to $L-1$, 
where $N^l_{p, d, j}$ is the maximum number of regions with $j$ active units in layer $l$  
from partitioning a space of dimension $d$ using $p$ hyperplanes. 

Note that 
there are at most $\binom{\mathcal{I}_l(k)}{j}$ regions defined by layer $l$ when $j$ unstable units are active and there are $k$ active units in layer $l-1$, which can be regarded as 
the subsets of $\mathcal{I}_l(k)$ units of size $j$. 
Since layer $l$ defines at most 
$\sum_{j=0}^{\min\{\mathcal{I}_l(k), d\}} \binom{\mathcal{I}_l(k)}{j}$ 
regions with an input dimension $d$ and $k$ active units in the layer above, 
by assuming the largest number of active hyperplanes among the unstable units  
and also using $\binom{\mathcal{I}_l(k)}{\mathcal{I}_l(k) - j} = \binom{\mathcal{I}_l(k)}{j}$, 
then we define $R(l, k, d)$ 
for $1 \leq l \leq L - 1$ 
as the following expression: 
$
    \displaystyle\sum_{j = 0}^{\min\{\mathcal{I}_l(k), d\}} \binom{\mathcal{I}_l(k)}{j} R(l + 1, \mathcal{A}_l(k) - j, \min\{\mathcal{A}_l(k) - j, d\}).
$

Without loss of generality, we assume that the input is generated by $n_0$ active units feeding the network, hence implying that the bound can be evaluated as $R(1,n_0,n_0)$:
$
\displaystyle\sum_{j_1=0}^{\min\{\mathcal{I}_1(k_0), d_1\}} \binom{\mathcal{I}_l(k_0)}{j_1} 
\cdots 
\sum_{j_L=0}^{\min\{\mathcal{I}_L(k_{L-1}), d_L\}} \binom{\mathcal{I}_L(k_{L-1})}{j_L},
$
where 
$k_0 = n_0$ and 
$k_l = \mathcal{A}_{l}(k_{l-1}) - j_{l-1}$ for $l \in \sL$, 
whereas 
$d_l = \min\{n_0, k_1, \ldots, k_{l-1}\}$. 
We obtain the final expression by nesting the values of $j_1, \ldots, j_L$.
\end{proof}

Theorem~\ref{thm:upper_bound} improves the result  
in~\cite{BoundingCounting} 
when not all hyperplanes partition every 
region from previous layers ($\mathcal{I}_l(k_{l-1}) < n_l$) 
or not all units can be active (smaller intervals for $j_l$ due to $\mathcal{A}_l(k_{l-1})$).


Now we discuss how the parameters that we introduced in this section 
can be computed exactly, or else approximated. We first bound the value of $\mathcal{I}_l(k)$. 
Let $U^-_l$ and $U^+_l$ denote the sets of inactive leaning and active leaning units in layer $l$, and $U_l = U^+_l \cup U^-_l$. 
For a given unit $i \in U^-_l$, 
we can define a set $J^-(l,i)$ of units from layer $l-1$ 
that, if active, 
can potentially make $i$ active. 
In fact, 
we can use the set in the summation of inequality~(\ref{eq:to_active}), 
and therefore let 
$J^-(l,i) := \{ j : 1 \leq j \leq n_{l-1}, \mW_{i j}^l > 0 \}$.  
For a given unit $i \in U^+_l$, 
we can similarly use the set in inequality~(\ref{eq:to_inactive}), 
and let $J^+(l,i) := \{ j : 1 \leq j \leq n_{l-1}, \mW_{i j}^l < 0 \}$. 
Conversely, 
let $I(l,j) := \{ i : i \in U^+_{l+1}, j \in J^+(l+1,i) \} \cup \{ i : i \in U^-_{l+1}, j \in J^-(l+1,i) \}$ 
be the 
units in layer $l+1$ that may be locally unstable if unit $j$ in layer $l$ is active. 

\begin{prp}
$\mathcal{I}_l(k) \leq$ $\max\limits_{S} \left\{ \left| \bigcup\limits_{j \in S} I(l-1,j)  \right| : S \in \sS^k \right\}$, where $\sS^k = \left\{ S : S \subseteq \{ 1, \ldots, n_{l-1} \}, |S| \leq k \right\}$.
\label{prp.Il}
\end{prp}

%
Next we bound the value of $\mathcal{A}_l(k)$ 
by considering 
a larger subset of the units in layer $l$ 
that only excludes locally inactive units. 
Let $n^+_l$ denote the number of stably active units in layer $l$, 
which is such that $n^+_l \leq n_l - |U_l|$, 
and 
let $I^-(l,j) :=$ $\{ i : i \in U^-_{l+1}, j \in J^-(l+1,i) \}$ be 
the 
inactive leaning units in layer $l+1$ that can be activated if unit $j$ in layer $l$ is active.

\begin{prp}
For $\sS^k$ defined as before, 
$\mathcal{A}_l(k) \leq n^+_l + |U^+_l| +$ $\max\limits_{S} \left\{ \left| \bigcup\limits_{j \in S} I^-(l-1,j)  \right| : S \in \sS^k \right\}$.
\label{prp.Al}
\end{prp}

In practice, however, we may only need to inspect a small number of such subsets. 
In the average-case analysis presented in Appendix~\ref{app:linear}, only $O(n_{l-1})$ subsets are needed. 
We also observed in the experiments that the minimum value of $k$ to maximize $\mathcal{I}_l(k)$ and  $\mathcal{A}_l(k)$ is rather small. 
%
If not, we 
can approximate 
$\mathcal{I}_l(k)$ and $\mathcal{A}_l(k)$ with 
strong optimality guarantees ($1-\frac{1}{e})$ using simple greedy algorithms for submodular function maximization~\cite{George_Nemhauser_MP1978}. 
We discuss that possibility in  Appendix~\ref{ap:submodular_bounds}.


\section{Experiments}

We tested on rectifier networks trained on the MNIST benchmark dataset
~\cite{MNIST}, 
consisting of 22 units distributed in two hidden layers and 10 output units, 
with 10 distinct networks for each distribution of units between the hidden layers. 
See Appendix~\ref{ap:exps} for more details about the networks and the implementation.
For each size of parity constraints $k$, which we denote as XOR-$k$, we measure the time to find the smallest coefficients $H_i^l$ and $\bar{H}_i^l$ for each unit along with the subsequent time of \texttt{MIPBound} (Algorithm~\ref{alg:mip_bound}). 
We let \texttt{MIPBound} run for enough steps to obtain a probability of 99.5\% in case all tested constraints of a given size preserve the formulation feasible, and we report the largest lower bound with probability at least 95\%. 
We also use the approach in~\cite{BoundingCounting} to count the exact number of regions for benchmarking. 
Since counting can be faster than sampling for smaller sets, 
we define a DNN with $\eta < 12$ as small and large otherwise. 
We use Configuration Upper Bound~(Configuration UB) for the bound in~\cite{BoundingCounting}.  
The upper bound from Theorem~\ref{thm:upper_bound}, 
which we denote as Empirical Upper Bound~(Empirical UB), 
is computed at a small fraction of the time to obtain coefficients $H_i^l$ and $\bar{H}_i^l$ for the lower bound. 
We denote as APP-$k$ the average between the XOR-$k$ Lower Bound (LB) and Empirical UB.

Table~\ref{tab:ub_gap} shows the gap closed by Empirical UB. 
Figure~\ref{fig:bounds}~(top) compares the bounds with the number of linear regions. 
Figure~\ref{fig:bounds}~(bottom) compares the time for exact counting and approximation. 
Figure~\ref{fig:accuracy} compares APP-$k$ with the accuracy of networks not having particularly narrow layers, 
in which case the number of regions relates to network accuracy~\cite{BoundingCounting}. 

\begin{table}[b]
    \centering
    \caption{Gap (\%) closed by Empirical UB between Configuration UB and the number of regions for widths $n_1;n_2;n_3$.}
    \label{tab:ub_gap}
    \vskip 0.15in
    \footnotesize
    \begin{tabular}{cccccccc}
        \textbf{Widths} & \textbf{Gap} & & \textbf{Widths} & \textbf{Gap} & & \textbf{Widths} & \textbf{Gap} \\
                 \cline{1-2} \cline{4-5} \cline{7-8}
         \noalign{\vskip2pt}
         1;21;10 & 73.1 & ~ & 8;14;10 & 0 & ~ & 15;7;10 & 0.5\\
         2;20;10 & 17.8 && 9;13;10 & 0 && 16;6;10 & 1 \\ 
         3;19;10 & 10.4 && 10;12;10 & 1 && 17;5;10 & 1.8 \\
         4;18;10 & 3.1 && 11;11;10 & 0 && 18;4;10 & 3.4 \\
         5;17;10 & 3.9 && 12;10;10 & 0 && 19;3;10 & 9.5 \\
         6;16;10 & 2 && 13;9;10 & 0.1 && 20;2;10 & 44.5 \\
         7;15;10 & 1.1 && 14;8;10 & 0.2 && 21;1;10 & 98.3 \\
    \end{tabular}
\end{table}

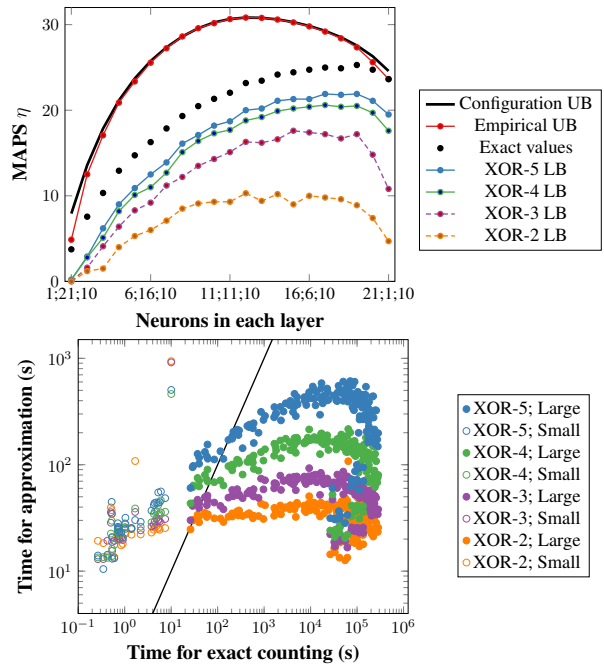
\begin{figure}[!ht]
  \centering
\pgfplotsset{every error bar/.style={ultra thick, color=color2b}}
\begin{tikzpicture}[scale=0.64]
\begin{axis}[filter discard warning=false, enlargelimits=false, enlarge x limits=0.02, 
title style={at={(0.5,-0.25)},anchor=north,yshift=-0.1},
xlabel={\Large \textbf{Neurons in each layer}},
symbolic x coords={
1;21;10,
2;20;10,
3;19;10,
4;18;10,
5;17;10,
6;16;10,
7;15;10,
8;14;10,
9;13;10,
10;12;10,
11;11;10,
12;10;10,
13;9;10,
14;8;10,
15;7;10,
16;6;10,
17;5;10,
18;4;10,
19;3;10,
20;2;10,
21;1;10,
},
xtick=\empty,
extra x ticks={1;21;10,6;16;10,11;11;10,16;6;10,21;1;10,},
ylabel={\Large \textbf{MAPS $\eta$}},ymin=0,ymax=32, 
every axis y label/.style={at={(ticklabel cs:0.5)},rotate=90,anchor=near ticklabel}, 
legend style={at={(1.35,0.7)},anchor=north}
]
\pgfplotstableread{
x         LB2 LB3 LB4 LB5 Actual UB
1;21;10 0 0 0.2 0.2 3.74578398935473 4.871661452
2;20;10 1.2 1.6 2.8 2.9 7.55927023280925 12.51012978
3;19;10 1.5 4.1 5.1 6.2 10.3297919944285 17.06857397
4;18;10 4 6.4 8.2 9 12.9304728099384 20.88539862
5;17;10 5.3 8.3 10.1 10.9 14.7301783839329 23.35997465
6;16;10 6 9.2 11 12.5 16.2731803615759 25.55426499
7;15;10 7.1 11.2 12.7 13.9 17.8701232759008 27.24137191
8;14;10 8.5 12.2 15.1 16.1 19.3202207206123 28.61644505
9;13;10 9.1 13.5 16.4 17.1 20.4784437009778 29.57705411
10;12;10 9.3 14.3 17.3 18.2 21.3292346036261 30.16964856
11;11;10 9.3 15.1 17.7 18.7 22.0421263485602 30.65206824
12;10;10 10.3 16.3 18.8 20 23.1744243554702 30.82179686
13;9;10 9.4 16.2 19.2 20.2 23.4493166229282 30.78711069
14;8;10 10.2 16.6 19.9 21.1 24.1649542943734 30.59319168
15;7;10 9 17.6 20.2 21.3 24.4277197234678 30.25915542
16;6;10 10 17.4 20.4 21.3 24.734419015675 29.79429629
17;5;10 9.8 17.2 20.6 21.9 24.9841964977662 29.18751275
18;4;10 9.6 16.7 20.4 21.8 24.8980204475814 28.39745962
19;3;10 8.9 17.2 20.5 21.9 25.2842687073009 27.34556271
20;2;10 7.4 14.8 19.7 21.1 24.7355346400191 25.60603956
21;1;10 4.7 10.8 17.6 19.5 23.6276496907593 23.64355125
}{\myLBtable}
\addplot[ultra thick,color=black] coordinates {
(1;21;10,7.9248125036)
(2;20;10,13.5839054522)
(3;19;10,17.853973481)
(4;18;10,21.1436454091)
(5;17;10,23.7136171495)
(6;16;10,25.7422203244)
(7;15;10,27.350022394)
(8;14;10,28.6165002479)
(9;13;10,29.5770783425)
(10;12;10,30.2569334492)
(11;11;10,30.6546206004)
(12;10;10,30.8254896699)
(13;9;10,30.7918099495)
(14;8;10,30.6078270702)
(15;7;10,30.2891810339)
(16;6;10,29.8464014684)
(17;5;10,29.2663182587)
(18;4;10,28.5213465663)
(19;3;10,27.562042845)
(20;2;10,26.3037633669)
(21;1;10,24.5849618701)
};
\addplot+[thick, mark=*,mark size=1.5,color=color1] 
  plot[thick, error bars/.cd, y dir=both, y explicit]
  table[x=x,y=UB] {\myLBtable};
\addplot[only marks, mark=*,mark size=1.5,color=black] 
  plot[thick, error bars/.cd, y dir=both, y explicit]
  table[x=x,y=Actual] {\myLBtable};
\addplot+[thick, mark=*,mark size=1.5,color=color2] 
  plot[thick, error bars/.cd, y dir=both, y explicit]
  table[x=x,y=LB5] {\myLBtable};
\addplot+[thick, mark=*,mark size=1.5,color=color3] 
  plot[thick, error bars/.cd, y dir=both, y explicit]
  table[x=x,y=LB4] {\myLBtable};
\addplot+[thick, mark=*,mark size=1.5,color=color4] 
  plot[thick, error bars/.cd, y dir=both, y explicit]
  table[x=x,y=LB3] {\myLBtable};
\addplot+[thick, mark=*,mark size=1.5,color=color5] 
  plot[thick, error bars/.cd, y dir=both, y explicit]
  table[x=x,y=LB2] {\myLBtable};
\legend{\large Configuration UB, \large Empirical UB, \large Exact values, \large XOR-5 LB, \large XOR-4 LB, \large XOR-3 LB, \large XOR-2 LB}
\end{axis}
\end{tikzpicture}
~~~
\begin{tikzpicture}[scale=0.64]
\begin{loglogaxis}[%
scatter/classes={%
    d1={mark=*,color=color2},
    d2={mark=o,color=color2},
    c1={mark=*,color=color3},
    c2={mark=o,color=color3},
    b1={mark=*,color=color4},
    b2={mark=o,color=color4},
    a1={mark=*,color=color5},
    a2={mark=o,color=color5}
    }, ymin=4, ymax=1.5e3, xmin=0.1, 
    xlabel={\Large \textbf{Time for exact counting (s)}},
    ylabel={\Large \textbf{Time for approximation (s)}},
    legend style={at={(1.35,0.8)},anchor=north}]
\addplot[scatter,only marks,%
    scatter src=explicit symbolic]%
table[meta=label] {
x y label

0.523985 37.41227 a2
0.750113 20.39725 a2
2.21071 22.13323 a2
58.7192 32.2739 a1
289.091 37.2055 a1
558.56 33.149 a1
1810.64 36.4536 a1
5099.98 43.9032 a1
9086.78 43.2874 a1
16847.7 43.8735 a1
39948.5 39.6082 a1
46983.8 38.6134 a1
50439.4 29.8445 a1
79843.2 35.5874 a1
83023.3 39.6421 a1
179140 39.8935 a1
257249 28.5591 a1
160641 29.7121 a1
192763 23.9198 a1
98753.8 20.42009 a1
41936.9 13.3819 a1
0.34674 18.33418 a2
0.711875 24.89808 a2
6.01416 28.6179 a2
61.4436 34.2519 a1
262.848 36.5129 a1
344.769 32.9911 a1
2610.26 38.8657 a1
3524.94 38.2715 a1
7306.76 41.3585 a1
4806.17 32.3313 a1
19296.1 34.7248 a1
50413.7 46.1903 a1
77955.3 38.0241 a1
151930 42.8835 a1
92748.2 29.0386 a1
113161 32.4218 a1
141581 38.8863 a1
107941 40.1616 a1
243487 25.0923 a1
136113 17.89424 a1
43119.4 18.56303 a1
0.505154 39.733 a2
0.713301 20.45856 a2
5.23014 25.58737 a2
29.8465 30.31732 a1
217.645 34.2788 a1
599.031 32.7052 a1
1662.48 39.3269 a1
4415.62 40.1353 a1
11136.6 44.2342 a1
12801.3 41.0983 a1
31301.6 40.8586 a1
50211.3 39.4847 a1
70893.6 33.1357 a1
155138 37.9937 a1
93328 37.5232 a1
180351 28.3931 a1
151570 34.0629 a1
160094 35.826 a1
271692 25.407 a1
76404.6 21.7147 a1
32286.7 16.40233 a1
0.535543 21.0332 a2
1.39969 22.02606 a2
5.66663 24.50358 a2
39.9116 28.7588 a1
170.562 32.4137 a1
435.396 34.1192 a1
2306.31 34.184 a1
4503.86 38.0863 a1
9892.26 37.4725 a1
18249 41.4109 a1
22724.8 39.3619 a1
55690.9 48.4515 a1
91885.6 30.5329 a1
89001.6 31.4354 a1
91086.2 34.0416 a1
208890 35.7167 a1
244444 35.9283 a1
236372 32.6224 a1
223123 21.1453 a1
99843.2 18.49884 a1
61792.1 13.40483 a1
1.67956 108.34816 a2
0.910853 18.04294 a2
3.77542 23.05085 a2
40.1022 32.6575 a1
205.432 34.9206 a1
572.159 31.4086 a1
1642.13 33.1279 a1
4503.89 39.4653 a1
11596.9 41.7796 a1
23435.8 41.5682 a1
15968 32.6274 a1
46428.4 45.1964 a1
58183.7 39.4351 a1
103388 30.6525 a1
221887 39.7836 a1
115735 42.7033 a1
199892 29.9154 a1
178872 25.7385 a1
279029 23.3706 a1
90260.8 18.31344 a1
40707.9 20.62233 a1
0.538518 34.69654 a2
0.838036 17.46839 a2
7.43046 25.41833 a2
28.9055 29.9289 a1
133.31 28.9616 a1
1045.23 41.6371 a1
2232.75 41.3857 a1
4180.16 40.1847 a1
7947.36 35.8701 a1
16612.1 43.7274 a1
37630 41.7051 a1
71146.5 38.7682 a1
111095 29.30531 a1
87219.5 31.4183 a1
88616.6 34.8373 a1
148153 33.1926 a1
145205 45.9859 a1
231745 30.5 a1
237898 24.7527 a1
105552 18.6228 a1
26454.1 14.38388 a1
0.598082 13.57509 a2
0.76804 19.6711 a2
5.27298 25.06705 a2
37.0087 33.8615 a1
205.595 32.6251 a1
594.899 33.1899 a1
1764.7 36.8542 a1
5061.31 42.7829 a1
7432.41 37.1685 a1
17154.2 42.7636 a1
37107.3 34.4812 a1
55171.2 42.9657 a1
63420.4 30.6738 a1
116169 33.2679 a1
165249 40.7468 a1
104835 48.3922 a1
169712 44.421 a1
173439 24.6773 a1
176431 25.124 a1
112160 17.091 a1
33222.5 15.01488 a1
0.42482 13.77212 a2
1.01427 22.95337 a2
4.2449 24.46413 a2
60.5177 33.6749 a1
280.007 33.3976 a1
488.32 34.3588 a1
2916.75 39.8918 a1
3657.88 41.6615 a1
9321.74 50.5936 a1
17185.8 41.6065 a1
38195.8 37.1422 a1
64622.5 107.9193 a1
75134.3 40.258 a1
91463.1 33.0598 a1
139974 44.3504 a1
128659 33.8705 a1
159162 29.8185 a1
226335 31.8448 a1
284091 25.6656 a1
134924 21.8666 a1
29609.1 43.00646 a1
0.265353 19.27836 a2
0.721769 21.17796 a2
4.39644 25.70935 a2
26.0898 24.69031 a1
84.7606 31.365 a1
475.25 35.0608 a1
968.125 30.9258 a1
5234.72 42.2911 a1
8693.03 40.9433 a1
13486.7 36.5096 a1
24188.2 37.2521 a1
45810.6 41.8726 a1
68342.4 30.4907 a1
128518 28.4241 a1
76166.2 29.298 a1
85726 38.5444 a1
114198 30.9336 a1
227964 39.0495 a1
242255 25.3497 a1
94370.5 20.84822 a1
55319.7 12.64584 a1
10.0393 942.00865 a2
1.0377 22.48391 a2
4.80232 26.36393 a2
53.3183 35.3884 a1
114.73 31.2591 a1
857.807 43.0333 a1
758.183 32.9547 a1
8089.28 38.9241 a1
11997.4 43.9194 a1
22700.3 42.4455 a1
35687.7 37.3316 a1
61652.8 45.777 a1
80391.4 34.0562 a1
111994 36.6646 a1
106684 31.944 a1
136547 32.0463 a1
135511 37.9001 a1
174890 34.3916 a1
203710 22.4674 a1
75109 18.73725 a1
33967.8 16.73236 a1

0.523985 34.53117 b2
0.750113 21.34348 b2
2.21071 24.20166 b2
58.7192 41.6446 b1
289.091 52.6055 b1
558.56 55.2306 b1
1810.64 65.1395 b1
5099.98 72.7435 b1
9086.78 78.0011 b1
16847.7 79.9659 b1
39948.5 69.052 b1
46983.8 68.0132 b1
50439.4 52.777 b1
79843.2 61.783 b1
83023.3 61.2645 b1
179140 52.3364 b1
257249 48.9978 b1
160641 49.5463 b1
192763 37.7058 b1
98753.8 25.4545 b1
41936.9 19.19663 b1
0.34674 13.0003 b2
0.711875 24.9112 b2
6.01416 35.6173 b2
61.4436 46.116 b1
262.848 58.9588 b1
344.769 48.4726 b1
2610.26 70.4473 b1
3524.94 76.1002 b1
7306.76 77.4676 b1
4806.17 57.6569 b1
19296.1 65.7335 b1
50413.7 86.2904 b1
77955.3 64.1184 b1
151930 71.2775 b1
92748.2 53.2482 b1
113161 49.0138 b1
141581 58.7939 b1
107941 52.5051 b1
243487 34.8275 b1
136113 22.90674 b1
43119.4 22.05562 b1
0.505154 28.5222 b2
0.713301 22.94224 b2
5.23014 29.7491 b2
29.8465 43.2635 b1
217.645 50.6435 b1
599.031 50.3611 b1
1662.48 67.1275 b1
4415.62 77.1238 b1
11136.6 79.3097 b1
12801.3 80.7533 b1
31301.6 68.3438 b1
50211.3 65.9904 b1
70893.6 61.0623 b1
155138 69.5868 b1
93328 61.7768 b1
180351 57.0902 b1
151570 51.6221 b1
160094 45.2454 b1
271692 36.6554 b1
76404.6 25.8122 b1
32286.7 18.79175 b1
0.535543 18.8092 b2
1.39969 25.4754 b2
5.66663 29.3107 b2
39.9116 39.2039 b1
170.562 49.0974 b1
435.396 53.884 b1
2306.31 57.9535 b1
4503.86 68.0659 b1
9892.26 70.1217 b1
18249 69.9485 b1
22724.8 67.7287 b1
55690.9 75.2419 b1
91885.6 61.1714 b1
89001.6 59.4302 b1
91086.2 51.9407 b1
208890 56.4021 b1
244444 51.0656 b1
236372 43.1146 b1
223123 33.8818 b1
99843.2 25.3606 b1
61792.1 16.66066 b1
1.67956 35.99196 b2
0.910853 19.62708 b2
3.77542 26.31096 b2
40.1022 42.9611 b1
205.432 57.8914 b1
572.159 51.4098 b1
1642.13 60.0213 b1
4503.89 74.2771 b1
11596.9 78.808 b1
23435.8 72.1956 b1
15968 52.0238 b1
46428.4 69.0086 b1
58183.7 60.4604 b1
103388 53.3433 b1
221887 59.3975 b1
115735 60.0358 b1
199892 48.4243 b1
178872 40.3399 b1
279029 34.2507 b1
90260.8 23.73828 b1
40707.9 25.8171 b1
0.538518 15.52554 b2
0.838036 20.85393 b2
7.43046 30.9014 b2
28.9055 41.0264 b1
133.31 38.1837 b1
1045.23 71.4714 b1
2232.75 78.9957 b1
4180.16 65.5807 b1
7947.36 65.8788 b1
16612.1 83.9194 b1
37630 80.8422 b1
71146.5 73.8616 b1
111095 57.13431 b1
87219.5 63.6322 b1
88616.6 65.2801 b1
148153 51.5609 b1
145205 57.9617 b1
231745 43.2349 b1
237898 37.2542 b1
105552 25.2931 b1
26454.1 22.23126 b1
0.598082 19.27161 b2
0.76804 20.08918 b2
5.27298 30.4207 b2
37.0087 46.9248 b1
205.595 51.1002 b1
594.899 58.977 b1
1764.7 66.7356 b1
5061.31 77.5625 b1
7432.41 69.8451 b1
17154.2 76.2665 b1
37107.3 60.4272 b1
55171.2 62.2163 b1
63420.4 57.3265 b1
116169 54.8797 b1
165249 55.608 b1
104835 67.0391 b1
169712 56.9882 b1
173439 37.3533 b1
176431 36.8134 b1
112160 22.99359 b1
33222.5 17.05103 b1
0.42482 19.33593 b2
1.01427 25.84128 b2
4.2449 27.05284 b2
60.5177 51.3219 b1
280.007 51.5801 b1
488.32 56.9349 b1
2916.75 71.4128 b1
3657.88 68.2816 b1
9321.74 92.8361 b1
17185.8 78.3903 b1
38195.8 68.9702 b1
64622.5 136.0172 b1
75134.3 65.4553 b1
91463.1 58.9961 b1
139974 59.7583 b1
128659 56.8463 b1
159162 49.5776 b1
226335 48.2671 b1
284091 38.1295 b1
134924 25.7455 b1
29609.1 47.1192 b1
0.265353 12.96773 b2
0.721769 23.16009 b2
4.39644 29.204 b2
26.0898 30.0876 b1
84.7606 44.1473 b1
475.25 55.8049 b1
968.125 58.6978 b1
5234.72 74.3317 b1
8693.03 74.7647 b1
13486.7 68.4426 b1
24188.2 66.7989 b1
45810.6 73.4624 b1
68342.4 62.7833 b1
128518 51.3233 b1
76166.2 56.5387 b1
85726 58.4424 b1
114198 48.8458 b1
227964 43.1264 b1
242255 35.373 b1
94370.5 25.3645 b1
55319.7 19.07175 b1
10.0393 915.94065 b2
1.0377 24.96587 b2
4.80232 31.1052 b2
53.3183 49.6945 b1
114.73 41.0082 b1
857.807 74.1013 b1
758.183 48.7678 b1
8089.28 72.8041 b1
11997.4 77.9766 b1
22700.3 74.2419 b1
35687.7 67.0663 b1
61652.8 77.772 b1
80391.4 72.354 b1
111994 56.2492 b1
106684 52.7008 b1
136547 59.8461 b1
135511 59.8406 b1
174890 46.8406 b1
203710 33.2626 b1
75109 24.1304 b1
33967.8 22.6093 b1

0.523985 29.78027 c2
0.750113 23.06955 c2
2.21071 26.1302 c2
58.7192 72.2341 c1
289.091 103.8543 c1
558.56 110.3675 c1
1810.64 136.5235 c1
5099.98 175.6644 c1
9086.78 169.4757 c1
16847.7 195.1226 c1
39948.5 169.862 c1
46983.8 182.0784 c1
50439.4 157.9252 c1
79843.2 137.4435 c1
83023.3 134.79 c1
179140 99.4423 c1
257249 107.985 c1
160641 113.4489 c1
192763 74.7413 c1
98753.8 41.3055 c1
41936.9 22.99173 c1
0.34674 13.55979 c2
0.711875 25.98493 c2
6.01416 46.2303 c2
61.4436 69.4005 c1
262.848 111.5111 c1
344.769 80.9091 c1
2610.26 145.7986 c1
3524.94 162.9658 c1
7306.76 158.5242 c1
4806.17 129.4493 c1
19296.1 155.5369 c1
50413.7 174.6126 c1
77955.3 164.8897 c1
151930 158.1167 c1
92748.2 120.5616 c1
113161 97.8051 c1
141581 122.6181 c1
107941 84.4012 c1
243487 66.1322 c1
136113 35.87094 c1
43119.4 26.0079 c1
0.505154 22.5553 c2
0.713301 24.35952 c2
5.23014 33.7545 c2
29.8465 63.7701 c1
217.645 98.7034 c1
599.031 91.3936 c1
1662.48 157.7163 c1
4415.62 134.9286 c1
11136.6 175.9285 c1
12801.3 176.7265 c1
31301.6 163.9077 c1
50211.3 145.7142 c1
70893.6 155.4197 c1
155138 180.0009 c1
93328 137.114 c1
180351 128.629 c1
151570 126.1446 c1
160094 85.4903 c1
271692 68.1769 c1
76404.6 41.4013 c1
32286.7 23.0903 c1
0.535543 16.1951 c2
1.39969 24.80997 c2
5.66663 34.2906 c2
39.9116 62.5649 c1
170.562 96.9867 c1
435.396 106.4959 c1
2306.31 115.1723 c1
4503.86 177.9003 c1
9892.26 170.5099 c1
18249 159.917 c1
22724.8 161.3994 c1
55690.9 183.6923 c1
91885.6 154.2988 c1
89001.6 139.8138 c1
91086.2 138.7799 c1
208890 155.8641 c1
244444 126.5804 c1
236372 88.2877 c1
223123 74.2995 c1
99843.2 37.4358 c1
61792.1 21.97606 c1
1.67956 38.81486 c2
0.910853 20.4743 c2
3.77542 28.7592 c2
40.1022 70.9175 c1
205.432 100.4265 c1
572.159 104.4518 c1
1642.13 121.4348 c1
4503.89 153.6264 c1
11596.9 172.502 c1
23435.8 170.2653 c1
15968 125.8666 c1
46428.4 169.7745 c1
58183.7 132.0276 c1
103388 137.6862 c1
221887 148.8648 c1
115735 129.5294 c1
199892 117.6672 c1
178872 89.7864 c1
279029 76.1438 c1
90260.8 36.35438 c1
40707.9 29.7326 c1
0.538518 13.89363 c2
0.838036 22.25616 c2
7.43046 36.1558 c2
28.9055 61.4617 c1
133.31 63.9166 c1
1045.23 139.7993 c1
2232.75 168.8228 c1
4180.16 136.3321 c1
7947.36 162.901 c1
16612.1 175.4205 c1
37630 174.5718 c1
71146.5 200.8613 c1
111095 138.66511 c1
87219.5 140.6271 c1
88616.6 167.4624 c1
148153 114.2262 c1
145205 128.7579 c1
231745 76.6325 c1
237898 79.951 c1
105552 39.2515 c1
26454.1 24.08906 c1
0.598082 17.65871 c2
0.76804 22.07503 c2
5.27298 41.7513 c2
37.0087 75.1826 c1
205.595 88.5753 c1
594.899 122.1529 c1
1764.7 148.127 c1
5061.31 179.2949 c1
7432.41 163.3146 c1
17154.2 162.5585 c1
37107.3 149.8068 c1
55171.2 149.7524 c1
63420.4 131.9012 c1
116169 141.3581 c1
165249 130.9916 c1
104835 136.2671 c1
169712 112.7528 c1
173439 74.0333 c1
176431 68.1366 c1
112160 31.15719 c1
33222.5 20.94128 c1
0.42482 13.33487 c2
1.01427 25.71982 c2
4.2449 34.4855 c2
60.5177 90.1251 c1
280.007 109.2818 c1
488.32 115.178 c1
2916.75 175.5842 c1
3657.88 137.5294 c1
9321.74 214.8241 c1
17185.8 197.0381 c1
38195.8 157.9734 c1
64622.5 216.5399 c1
75134.3 141.4222 c1
91463.1 141.3064 c1
139974 154.2218 c1
128659 125.8665 c1
159162 122.9198 c1
226335 113.849 c1
284091 80.2622 c1
134924 40.854 c1
29609.1 52.5061 c1
0.265353 13.4659 c2
0.721769 23.56119 c2
4.39644 36.0791 c2
26.0898 43.3554 c1
84.7606 68.5531 c1
475.25 106.8274 c1
968.125 109.4267 c1
5234.72 166.0214 c1
8693.03 179.7572 c1
13486.7 171.3711 c1
24188.2 137.7392 c1
45810.6 192.3894 c1
68342.4 158.8761 c1
128518 140.4754 c1
76166.2 148.1642 c1
85726 124.9893 c1
114198 102.9617 c1
227964 87.4688 c1
242255 61.8382 c1
94370.5 36.645 c1
55319.7 20.48595 c1
10.0393 462.41665 c2
1.0377 27.5181 c2
4.80232 36.6874 c2
53.3183 80.3457 c1
114.73 71.2448 c1
857.807 131.4198 c1
758.183 87.9319 c1
8089.28 167.0115 c1
11997.4 181.8237 c1
22700.3 189.1501 c1
35687.7 139.237 c1
61652.8 186.9226 c1
80391.4 158.416 c1
111994 146.6927 c1
106684 134.6657 c1
136547 142.5987 c1
135511 111.4017 c1
174890 97.7738 c1
203710 72.1331 c1
75109 32.6329 c1
33967.8 24.4561 c1

0.523985 44.79117 d2
0.750113 25.6307 d2
2.21071 29.0258 d2
58.7192 119.5547 d1
289.091 210.0074 d1
558.56 237.0911 d1
1810.64 303.8925 d1
5099.98 432.8454 d1
9086.78 452.3737 d1
16847.7 521.1696 d1
39948.5 551.913 d1
46983.8 496.3724 d1
50439.4 467.9072 d1
79843.2 348.5485 d1
83023.3 335.72 d1
179140 334.6038 d1
257249 324.01 d1
160641 282.8845 d1
192763 214.969 d1
98753.8 99.6933 d1
41936.9 31.20083 d1
0.34674 10.4614 d2
0.711875 28.1973 d2
6.01416 56.0553 d2
61.4436 113.9511 d1
262.848 213.6331 d1
344.769 143.0897 d1
2610.26 316.7756 d1
3524.94 363.4988 d1
7306.76 431.4452 d1
4806.17 286.6103 d1
19296.1 445.6189 d1
50413.7 604.6036 d1
77955.3 433.8257 d1
151930 453.9637 d1
92748.2 384.0876 d1
113161 266.2219 d1
141581 413.2707 d1
107941 199.4981 d1
243487 158.9967 d1
136113 70.02454 d1
43119.4 33.3107 d1
0.505154 35.758 d2
0.713301 26.0391 d2
5.23014 43.7552 d2
29.8465 97.4492 d1
217.645 203.8169 d1
599.031 160.0454 d1
1662.48 348.4283 d1
4415.62 351.8216 d1
11136.6 543.3105 d1
12801.3 361.8845 d1
31301.6 447.7487 d1
50211.3 398.9492 d1
70893.6 428.7347 d1
155138 541.0849 d1
93328 376.659 d1
180351 387.808 d1
151570 355.3476 d1
160094 205.2803 d1
271692 172.639 d1
76404.6 72.4293 d1
32286.7 31.12 d1
0.535543 20.715 d2
1.39969 25.1199 d2
5.66663 44.5316 d2
39.9116 102.2036 d1
170.562 171.2167 d1
435.396 231.4941 d1
2306.31 261.1175 d1
4503.86 464.0213 d1
9892.26 421.6049 d1
18249 377.944 d1
22724.8 402.7074 d1
55690.9 509.8653 d1
91885.6 452.4758 d1
89001.6 408.9618 d1
91086.2 395.1169 d1
208890 435.5481 d1
244444 324.8184 d1
236372 214.3837 d1
223123 176.1006 d1
99843.2 72.2735 d1
61792.1 33.77316 d1
1.67956 30.20026 d2
0.910853 20.67305 d2
3.77542 37.0577 d2
40.1022 116.586 d1
205.432 175.5798 d1
572.159 220.2351 d1
1642.13 279.4098 d1
4503.89 356.7504 d1
11596.9 513.846 d1
23435.8 480.4383 d1
15968 295.1646 d1
46428.4 472.2965 d1
58183.7 424.9236 d1
103388 484.5462 d1
221887 440.5178 d1
115735 398.8344 d1
199892 320.3322 d1
178872 229.1628 d1
279029 148.3234 d1
90260.8 86.78848 d1
40707.9 36.5889 d1
0.538518 13.08734 d2
0.838036 23.2979 d2
7.43046 48.4737 d2
28.9055 94.761 d1
133.31 113.5686 d1
1045.23 287.6843 d1
2232.75 349.8608 d1
4180.16 312.9841 d1
7947.36 415.453 d1
16612.1 418.8835 d1
37630 407.7528 d1
71146.5 610.9293 d1
111095 464.12511 d1
87219.5 405.7561 d1
88616.6 603.0824 d1
148153 295.9281 d1
145205 274.5999 d1
231745 199.1485 d1
237898 189.0237 d1
105552 92.1218 d1
26454.1 31.46186 d1
0.598082 13.22345 d2
0.76804 23.26873 d2
5.27298 54.9154 d2
37.0087 126.4438 d1
205.595 185.1442 d1
594.899 245.6199 d1
1764.7 298.972 d1
5061.31 383.4759 d1
7432.41 400.6286 d1
17154.2 430.4285 d1
37107.3 396.6928 d1
55171.2 409.2404 d1
63420.4 391.6902 d1
116169 431.5381 d1
165249 371.9306 d1
104835 412.7681 d1
169712 289.1282 d1
173439 162.9211 d1
176431 180.5572 d1
112160 57.69219 d1
33222.5 30.58248 d1
0.42482 14.25658 d2
1.01427 28.0974 d2
4.2449 41.7859 d2
60.5177 154.9267 d1
280.007 190.0116 d1
488.32 238.6968 d1
2916.75 378.0712 d1
3657.88 397.8314 d1
9321.74 480.6371 d1
17185.8 584.1301 d1
38195.8 477.6834 d1
64622.5 500.0729 d1
75134.3 468.5692 d1
91463.1 404.5954 d1
139974 453.3478 d1
128659 344.3985 d1
159162 345.9478 d1
226335 277.6906 d1
284091 198.1373 d1
134924 94.651 d1
29609.1 60.3235 d1
0.265353 14.02329 d2
0.721769 24.44601 d2
4.39644 43.6671 d2
26.0898 60.5155 d1
84.7606 112.1012 d1
475.25 225.9769 d1
968.125 263.6787 d1
5234.72 369.0664 d1
8693.03 421.6142 d1
13486.7 478.2231 d1
24188.2 407.9932 d1
45810.6 493.3744 d1
68342.4 439.7901 d1
128518 450.7424 d1
76166.2 429.1352 d1
85726 325.4913 d1
114198 286.215 d1
227964 186.3009 d1
242255 157.9488 d1
94370.5 68.4432 d1
55319.7 28.30995 d1
10.0393 504.53965 d2
1.0377 31.7626 d2
4.80232 44.2309 d2
53.3183 132.0024 d1
114.73 133.8976 d1
857.807 314.1448 d1
758.183 199.788 d1
8089.28 364.1115 d1
11997.4 416.8537 d1
22700.3 520.4891 d1
35687.7 404.942 d1
61652.8 553.0136 d1
80391.4 555.409 d1
111994 428.4287 d1
106684 371.0897 d1
136547 363.5307 d1
135511 267.6095 d1
174890 255.8241 d1
203710 170.3564 d1
75109 62.2127 d1
33967.8 31.0484 d1
    };
\addplot[thick,color=black] coordinates {(0.1,0.1)(2e3,2e3)};
\legend{\large XOR-5; Large,\large XOR-5; Small,\large XOR-4; Large,\large XOR-4; Small,\large XOR-3; Large,\large XOR-3; Small,\large XOR-2; Large,\large XOR-2; Small}
\end{loglogaxis}
\end{tikzpicture}
\caption{Top: Averages of the Empirical UB and XOR-$k$ LBs with probability 95\% compared with the Configuration UB and the exact number of regions 
for 10 networks of each type. Bottom: Comparison of approximation vs. exact counting times by XOR size and number of regions.
%
}
\label{fig:bounds}
\end{figure}


\begin{figure}[!ht]
  \centering
\begin{tikzpicture}[scale=0.62]
\begin{axis}[filter discard warning=false, enlargelimits=false, enlarge x limits=0.02, 
title style={at={(0.5,-0.25)},anchor=north,yshift=-0.1},
xlabel={\large \textbf{MAPS $\eta$}},
ylabel={\large \textbf{Training error (CE)}},
every axis y label/.style={at={(ticklabel cs:0.5)},rotate=90,anchor=near ticklabel}, 
legend style={at={(0.55,0.35)},anchor=north},
width=5.5cm
]
\addplot[scatter,only marks,mark options={scale=0.4},scatter src=explicit symbolic,scatter/classes={
            x={mark=*,color=black}
        }] table [x index=4, y index=5, col sep=comma,meta index = 7] {benchmark.txt} ;
\addplot[line width=2pt,color=black] table [x index=4, y={create col/linear regression={y=tr}},col sep=comma] {benchmark.txt} ;
\addplot[scatter,only marks,mark options={scale=0.4},scatter src=explicit symbolic,scatter/classes={
            x={mark=*,color=color5}
        }] table [x index=0, y index=5, col sep=comma,meta index = 7] {benchmark.txt} ;
\addplot[line width=2pt,color=color5] table [x index=0, y={create col/linear regression={y=tr}},col sep=comma] {benchmark.txt} ;
\end{axis}
\end{tikzpicture}
~
\begin{tikzpicture}[scale=0.62]
\begin{axis}[filter discard warning=false, enlargelimits=false, enlarge x limits=0.02, 
title style={at={(0.5,-0.25)},anchor=north,yshift=-0.1},
xlabel={\large \textbf{MAPS $\eta$}},
ylabel={\large \textbf{Test error (\%MR)}},
every axis y label/.style={at={(ticklabel cs:0.5)},rotate=90,anchor=near ticklabel}, 
legend style={at={(1.32,0.7)},anchor=north},
width=5.5cm
]
\addplot[line width=2pt,color=black] table [x index=4, y={create col/linear regression={y=ts}}, col sep=comma] {benchmark.txt} ;
\addplot[line width=2pt,color=color5] table [x index=0, y={create col/linear regression={y=ts}}, col sep=comma] {benchmark.txt} ;
\addplot[scatter,only marks,mark options={scale=0.4},scatter src=explicit symbolic,scatter/classes={
            x={mark=*,color=black}
        }] table [x index=4, y index=6, col sep=comma,meta index = 7] {benchmark.txt} ;
\addplot[scatter,only marks,mark options={scale=0.4},scatter src=explicit symbolic,scatter/classes={
            x={mark=*,color=color5}
        }] table [x index=0, y index=6, col sep=comma,meta index = 7] {benchmark.txt} ;
\legend{\Large Exact, \Large APP-2}
\end{axis}
\end{tikzpicture}
~
\begin{tikzpicture}[scale=0.62]
\begin{axis}[filter discard warning=false, enlargelimits=false, enlarge x limits=0.02, 
title style={at={(0.5,-0.25)},anchor=north,yshift=-0.1},
xlabel={\large \textbf{MAPS $\eta$}},
ylabel={\large \textbf{Training error (CE)}},
every axis y label/.style={at={(ticklabel cs:0.5)},rotate=90,anchor=near ticklabel}, 
legend style={at={(0.55,0.35)},anchor=north},
width=5.5cm
]
\addplot[scatter,only marks,mark options={scale=0.4},scatter src=explicit symbolic,scatter/classes={
            x={mark=*,color=black}
        }] table [x index=4, y index=5, col sep=comma,meta index = 7] {benchmark.txt} ;
\addplot[line width=2pt,color=black] table [x index=4, y={create col/linear regression={y=tr}},col sep=comma] {benchmark.txt} ;
\addplot[scatter,only marks,mark options={scale=0.4},scatter src=explicit symbolic,scatter/classes={
            x={mark=*,color=color4}
        }] table [x index=1, y index=5, col sep=comma,meta index = 7] {benchmark.txt} ;
\addplot[line width=2pt,color=color4] table [x index=1, y={create col/linear regression={y=tr}},col sep=comma] {benchmark.txt} ;
\end{axis}
\end{tikzpicture}
~
\begin{tikzpicture}[scale=0.62]
\begin{axis}[filter discard warning=false, enlargelimits=false, enlarge x limits=0.02, 
title style={at={(0.5,-0.25)},anchor=north,yshift=-0.1},
xlabel={\large \textbf{MAPS $\eta$}},
ylabel={\large \textbf{Test error (\%MR)}},
every axis y label/.style={at={(ticklabel cs:0.5)},rotate=90,anchor=near ticklabel}, 
legend style={at={(1.32,0.7)},anchor=north},
width=5.5cm
]
\addplot[line width=2pt,color=black] table [x index=4, y={create col/linear regression={y=ts}}, col sep=comma] {benchmark.txt} ;
\addplot[line width=2pt,color=color4] table [x index=1, y={create col/linear regression={y=ts}}, col sep=comma] {benchmark.txt} ;
\addplot[scatter,only marks,mark options={scale=0.4},scatter src=explicit symbolic,scatter/classes={
            x={mark=*,color=black}
        }] table [x index=4, y index=6, col sep=comma,meta index = 7] {benchmark.txt} ;
\addplot[scatter,only marks,mark options={scale=0.4},scatter src=explicit symbolic,scatter/classes={
            x={mark=*,color=color4}
        }] table [x index=1, y index=6, col sep=comma,meta index = 7] {benchmark.txt} ;
\legend{\Large Exact, \Large APP-3}
\end{axis}
\end{tikzpicture}

\begin{tikzpicture}[scale=0.62]
\begin{axis}[filter discard warning=false, enlargelimits=false, enlarge x limits=0.02, 
title style={at={(0.5,-0.25)},anchor=north,yshift=-0.1},
xlabel={\large \textbf{MAPS $\eta$}},
ylabel={\large \textbf{Training error (CE)}},
every axis y label/.style={at={(ticklabel cs:0.5)},rotate=90,anchor=near ticklabel}, 
legend style={at={(0.55,0.35)},anchor=north},
width=5.5cm
]
\addplot[scatter,only marks,mark options={scale=0.4},scatter src=explicit symbolic,scatter/classes={
            x={mark=*,color=black}
        }] table [x index=4, y index=5, col sep=comma,meta index = 7] {benchmark.txt} ;
\addplot[line width=2pt,color=black] table [x index=4, y={create col/linear regression={y=tr}},col sep=comma] {benchmark.txt} ;
\addplot[scatter,only marks,mark options={scale=0.4},scatter src=explicit symbolic,scatter/classes={
            x={mark=*,color=color3}
        }] table [x index=2, y index=5, col sep=comma,meta index = 7] {benchmark.txt} ;
\addplot[line width=2pt,color=color3] table [x index=2, y={create col/linear regression={y=tr}},col sep=comma] {benchmark.txt} ;
\end{axis}
\end{tikzpicture}
~
\begin{tikzpicture}[scale=0.62]
\begin{axis}[filter discard warning=false, enlargelimits=false, enlarge x limits=0.02, 
title style={at={(0.5,-0.25)},anchor=north,yshift=-0.1},
xlabel={\large \textbf{MAPS $\eta$}},
ylabel={\large \textbf{Test error (\%MR)}},
every axis y label/.style={at={(ticklabel cs:0.5)},rotate=90,anchor=near ticklabel}, 
legend style={at={(1.32,0.7)},anchor=north},
width=5.5cm
]
\addplot[line width=2pt,color=black] table [x index=4, y={create col/linear regression={y=ts}}, col sep=comma] {benchmark.txt} ;
\addplot[line width=2pt,color=color3] table [x index=2, y={create col/linear regression={y=ts}}, col sep=comma] {benchmark.txt} ;
\addplot[scatter,only marks,mark options={scale=0.4},scatter src=explicit symbolic,scatter/classes={
            x={mark=*,color=black}
        }] table [x index=4, y index=6, col sep=comma,meta index = 7] {benchmark.txt} ;
\addplot[scatter,only marks,mark options={scale=0.4},scatter src=explicit symbolic,scatter/classes={
            x={mark=*,color=color3}
        }] table [x index=2, y index=6, col sep=comma,meta index = 7] {benchmark.txt} ;
\legend{\Large Exact, \Large APP-4}
\end{axis}
\end{tikzpicture}
~
\begin{tikzpicture}[scale=0.62]
\begin{axis}[filter discard warning=false, enlargelimits=false, enlarge x limits=0.02, 
title style={at={(0.5,-0.25)},anchor=north,yshift=-0.1},
xlabel={\large \textbf{MAPS $\eta$}},
ylabel={\large \textbf{Training error (CE)}},
every axis y label/.style={at={(ticklabel cs:0.5)},rotate=90,anchor=near ticklabel}, 
legend style={at={(0.55,0.35)},anchor=north},
width=5.5cm
]
\addplot[scatter,only marks,mark options={scale=0.4},scatter src=explicit symbolic,scatter/classes={
            x={mark=*,color=black}
        }] table [x index=4, y index=5, col sep=comma,meta index = 7] {benchmark.txt} ;
\addplot[line width=2pt,color=black] table [x index=4, y={create col/linear regression={y=tr}},col sep=comma] {benchmark.txt} ;
\addplot[scatter,only marks,mark options={scale=0.4},scatter src=explicit symbolic,scatter/classes={
            x={mark=*,color=color2}
        }] table [x index=3, y index=5, col sep=comma,meta index = 7] {benchmark.txt} ;
\addplot[line width=2pt,color=color2] table [x index=3, y={create col/linear regression={y=tr}},col sep=comma] {benchmark.txt} ;
\end{axis}
\end{tikzpicture}
~
\begin{tikzpicture}[scale=0.62]
\begin{axis}[filter discard warning=false, enlargelimits=false, enlarge x limits=0.02, 
title style={at={(0.5,-0.25)},anchor=north,yshift=-0.1},
xlabel={\large \textbf{MAPS $\eta$}},
ylabel={\large \textbf{Test error (\%MR)}},
every axis y label/.style={at={(ticklabel cs:0.5)},rotate=90,anchor=near ticklabel}, 
legend style={at={(1.32,0.7)},anchor=north},
width=5.5cm
]
\addplot[line width=2pt,color=black] table [x index=4, y={create col/linear regression={y=ts}}, col sep=comma] {benchmark.txt} ;
\addplot[line width=2pt,color=color2] table [x index=3, y={create col/linear regression={y=ts}}, col sep=comma] {benchmark.txt} ;
\addplot[scatter,only marks,mark options={scale=0.4},scatter src=explicit symbolic,scatter/classes={
            x={mark=*,color=black}
        }] table [x index=4, y index=6, col sep=comma,meta index = 7] {benchmark.txt} ;
\addplot[scatter,only marks,mark options={scale=0.4},scatter src=explicit symbolic,scatter/classes={
            x={mark=*,color=color2}
        }] table [x index=3, y index=6, col sep=comma,meta index = 7] {benchmark.txt} ;
\legend{\Large Exact, \Large APP-5}
\end{axis}
\end{tikzpicture}

\caption{Regression on network accuracy for exact and approximated count in black and color, respectively. The approximated count averages lower and upper bounds. As the XOR size increases, these regressions become more parallel, and thus the approximate count regression is more accurate.}
\label{fig:accuracy}
\end{figure}
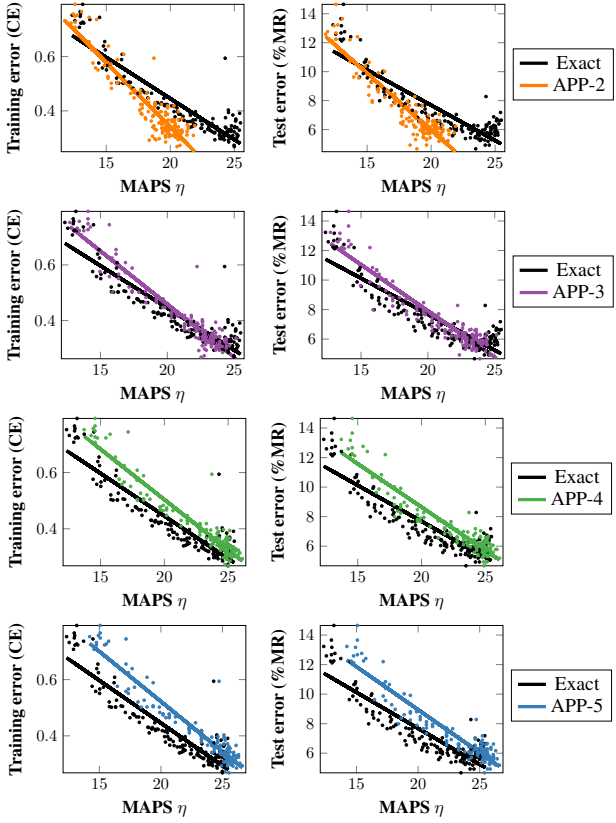

\section{Conclusion}

This paper introduced methods to obtain tighter bounds on the number of linear regions. 
These methods are considerably faster than direct enumeration, 
entail a probabilistic lower bound algorithm to count MILP solutions, 
and help understanding how ReLUs partition the input space.  


Prior work on bounding the number of linear regions has first focused on the benefit of depth~\cite{Pascanu,Montufar14}, and then on the bottleneck effect that is caused by a hidden layer that is too narrow~\cite{Montufar17} and more generally by small activation sets~\cite{BoundingCounting}. 
In our work, we looked further into how many units can possibly be active or not by taking into account the weights and the bias of each ReLU, which allows us to identify stable units. Stable units do not contribute as significantly to the number of linear regions. Consequently, we found that the bottleneck effect in the upper bound is even stronger in narrow layers, as evidenced in both extremes of Table~\ref{tab:ub_gap}.

The probabilistic lower bound is based on universal hashing functions to sample activation patterns, 
and 
more generally allows us to estimate the number of solutions on binary variables of MILP formulations. 
By exploiting callbacks that are typical of MILP solvers, 
the number of solver calls in the proposed algorithm \texttt{MIPBound} does not depend on the 
number of sizes for which we evaluate the probabilistic bounds 
like 
in related work. 
The algorithm is orders of magnitude faster than exact counting on networks with a large number of linear regions. 
These bounds can be parameterized for a balance between precision and speed. 
Nevertheless, we noted that 
lower bounds from XOR constraints of size 2, which are faster to compute but not as accurate, can be used to compare relative expressiveness since the curves from XOR-2 to XOR-5 have a very similar shape.


When upper and lower bounds are combined, 
the weakest approximation still preserves a negative correlation with accuracy, 
hence indicating that it may suffice to compare networks for  relative expressiveness. 
It is important to note that we do not need the exact count to compare two networks in terms of 
expressiveness or performance. Nevertheless, the stronger approximations produce more precise correlations, 
which is evidenced by the more parallel regressions and thus more stable gaps across network sizes. 




\appendix

\section{Convex Outer Approximation of a ReLU}\label{app:outer}

\begin{figure}[b]
\centering
\includegraphics[scale=0.19]{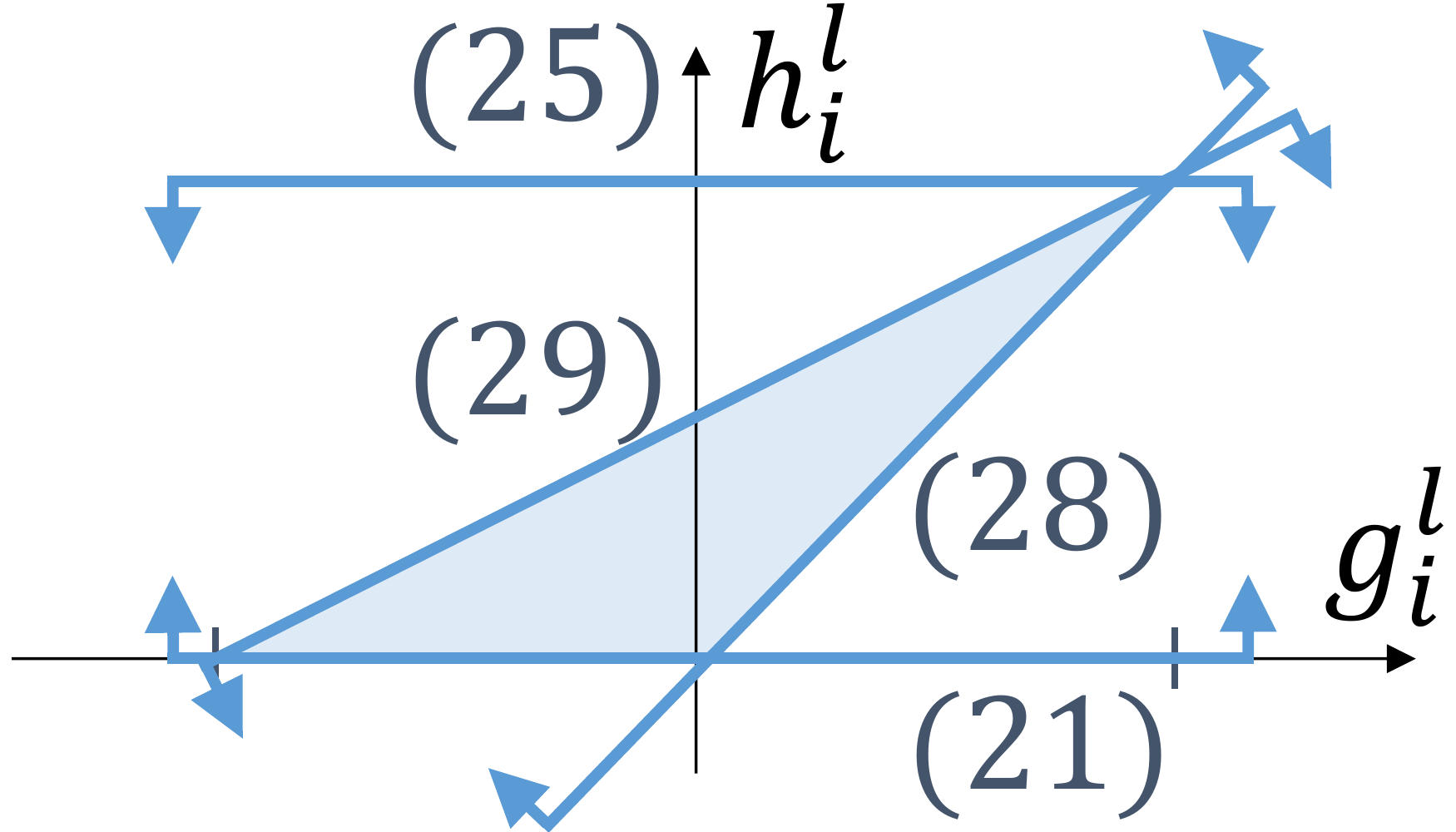}
\caption{
The projected inequalities (\ref{eq:h_pos}), (\ref{eq:h_ub}), (\ref{eq:slope_lb}), and (\ref{eq:upper_diagonal}) define the convex outer approximation on $(g_i^l, h_i^l)$.}\label{fig:projected}
\end{figure}

\begin{reflem}{lem:outer}
If $H_i^l = \arg \max_{h^{l-1}} \{ W_i^l h^{l-1} + b_i^l \} \geq 0$ and $\bar{H}_i^l = \arg \max_{h^{l-1}} \{ - W_i^l h^{l-1} - b_i^l \} \geq 0$, then the linear relaxation of (\ref{eq:mip_unit_begin})--(\ref{eq:mip_unit_end}) defines the convex outer approximation on $(g_i^l, h_i^l)$.
\end{reflem}
\begin{proof}
We begin with the linear relaxation of the formulation defined by constraints (\ref{eq:mip_after_Wb_begin})--(\ref{eq:mip_unit_end}): 
\begin{align}
    g_i^l = h_i^l - \bar{h}_i^l  & \qquad \Leftrightarrow \qquad \bar{h}_i^l = h_i^l - g_i^l \label{eq:hbar_linear} \\
    h_i^l \leq H_i^l z_i^l & \qquad \Leftrightarrow \qquad z_i^l \geq \frac{h_i^l}{H_i^l} \label{eq:H_ub} \\
    \bar{h}_i^l \leq \bar{H}_i^l (1-z_i^l) & \qquad \Leftrightarrow \qquad z_i^l \leq 1 - \frac{\bar{h}_i^l}{\bar{H}_i^l} \label{eq:H_lb} \\
    h_i^l \geq 0 \label{eq:h_pos} \\
    \bar{h}_i^l \geq 0 \label{eq:hbar_pos} \\
    0 \leq z_i^l \leq 1 \label{eq:z_01}
\end{align}
We first project $z_i^l$ out by isolating that variable on one side of each inequality, 
and then combining every lower bound with every upper bound. 
Hence, we replace (\ref{eq:H_ub}), (\ref{eq:H_lb}), and (\ref{eq:z_01}) with:
\begin{align}
    \frac{h_i^l}{H_i^l} \leq  1 - \frac{\bar{h}_i^l}{\bar{H}_i^l} & \qquad \Leftrightarrow \qquad \bar{h}_i^l \leq \bar{H}_i^l \left( 1 - \frac{h_i^l}{H_i^l} \right) \label{eq:hbar_clumsy} \\
    \frac{h_i^l}{H_i^l} \leq 1 & \qquad \Leftrightarrow \qquad h_i^l \leq H_i^l \label{eq:h_ub} \\
    0 \leq 1 - \frac{\bar{h}_i^l}{\bar{H}_i^l} & \qquad \Leftrightarrow \qquad \bar{h}_i^l \leq \bar{H}_i^l \label{eq:hbar_ub} \\
    0 \leq 1 \label{eq:taut}
\end{align}
Next, we project $\bar{h}_i^l$ through the same steps, also combining the equality with the lower and upper bounds on the variable. 
Hence, we replace (\ref{eq:hbar_linear}), (\ref{eq:hbar_pos}), (\ref{eq:hbar_clumsy}), and (\ref{eq:hbar_ub}) with:
\begin{align}
    h_i^l - g_i^l \geq 0 \label{eq:slope_lb} \\
    h_i^l - g_i^l \leq \bar{H}_i^l \left( 1 - \frac{h_i^l}{H_i^l} \right) \label{eq:upper_diagonal} \\
    h_i^l - g_i^l \leq \bar{H}_i^l \label{eq:slope_ub} \\
    \bar{H}_i^l \left( 1 - \frac{h_i^l}{H_i^l} \right) \geq 0 \label{eq:redund} \\
    \bar{H}_i^l \geq 0 \label{eq:hbar_ub_pos}
\end{align}
We drop (\ref{eq:taut}) as a tautology and (\ref{eq:hbar_ub_pos}) as implicit on our assumptions. Similarly, for $\bar{H}_i^l > 0$, inequality (\ref{eq:redund}) is equivalent to (\ref{eq:h_ub}). Therefore, we are left with (\ref{eq:h_pos}), (\ref{eq:h_ub}), (\ref{eq:slope_lb}), (\ref{eq:upper_diagonal}), and (\ref{eq:slope_ub}). 
We show in Figure~\ref{fig:projected} that the first four inequalities define the convex outer approximation on $(g_i^l, h_i^l)$, 
whereas (\ref{eq:slope_ub}) is active at $(-\bar{H}_i^l,0)$ and $(H_i^l,H_i^l+\bar{H}_i^l)$ and thus dominated by (\ref{eq:upper_diagonal}) in that region.
\end{proof}

%
%

\section{Parity Constraints on 0--1 Variables}\label{ap:xor}

Similarly to the case of SAT formulas, we need to find a suitable way of translating a XOR constraint to a MILP formulation. 
This is also discussed in 
~\cite{ErmonOpt} 
for a related application of probabilistic inference. 
Let $w$ be the set of binary variables and $U \subseteq V$ the set of indices of $w$ variables of a XOR constraint. 
To remove all assignments with an even sum in $U$, 
we use a family of canonical cuts on the unit hypercube~\cite{CanonicalCuts}:
\begin{equation}
    \sum_{i \in U'} w_i - \sum_{i \in U \setminus U'} w_i \leq |U'| - 1 ~ \forall U' \subseteq U : |U'| \text{ even},
\end{equation}
which is effectively separating each such assignment with one constraint. Although exponential in $k$, 
each of those constraints -- and only those -- are necessary to define a convex hull of the feasible assignments 
in the absence of other constraints~\cite{ParityPolytope}. 
However, we note that we can do better when $k=2$ by using
\begin{equation}
    w_i + w_j = 1 \qquad \text{ if U = \{ i, j \}}.
\end{equation}
%

\section{Deriving the Lower Bound Probabilities of Algorithm~\ref{alg:mip_bound}}\label{ap:probs}

The probabilities given to the lower bounds by Algorithm~\ref{alg:mip_bound} are due to the main result for MBound in~\cite{MBound}, 
which is based on the following parameters: 
XOR size $k$; number of restrictions $r$; loop repetitions $i$; number of repetitions that remain feasible after $j$ restrictions $f[j]$; 
deviation $\delta \in (0, 1/2]$; and precision slack $\alpha \geq 1$. 
We choose the values for the latter two. 

A strict lower bound of $2^{r-\alpha}$ can be defined 
if 
\begin{equation}\label{eq:lb_condition}
    f[j] \geq i.(1/2+\delta),
\end{equation}
and for $\delta \in (0,1/2)$ 
it holds with probability $1 - \left(\frac{e^{\beta}}{(1+\beta)^{1+\beta}}\right)^{i/2^\alpha}$ 
for $\beta = 2^\alpha.(1/2+\delta)-1$. 
We choose $\alpha = 1$, 
hence making $\beta = 2.\delta$, 
and then set $\delta$ to the largest value satisfying condition~(\ref{eq:lb_condition}). 

\section{Previous Bounds on the Maximum Number of Linear Regions}\label{app:bounds}

We can obtain a bound for deep networks by recursively combining the bounds obtained on each layer. 
By assuming that every linear region defined by the first $l-1$ layers 
is then subdivided into the maximum possible number of linear regions defined by the activation hyperplanes of layer $l$, 
we obtain the implicit bound of $\prod_{l=1}^L \sum_{j=0}^{n_{l-1}} \binom{n_l}{j}$ 
in \cite{Raghu}. 
By observing that the dimension of the input of layer $l$ on each linear region 
is also constrained by the smallest input dimension among layers $1$ to $l-1$, 
we obtain the bound in~\cite{Montufar17} 
of $\prod_{l=1}^L \sum_{j=0}^{d_l} \binom{n_l}{j}$, where $d_l = \min\{n_0, n_1, \ldots, n_l\}$. 
If we refine the effect on the input dimension by also considering that the number of units that are active on each layer varies across the linear regions, 
we obtain the tighter bound in 
\cite{BoundingCounting} 
of $\sum_{(j_1,\ldots,j_L) \in J} \prod_{l=1}^L \binom{n_l}{j_l}$, where $J = \{(j_1, \ldots, j_L) \in \mathbb{Z}^L: 0 \leq j_l \leq \min\{n_0, n_1 - j_1, \ldots, n_{l-1} - j_{l-1}, n_l\}\ \forall l \in \sL \}$.

\section{Average-Case Complexity of Computing $A_l(k)$ and $I_l(k)$ for Any $k$}\label{app:linear}

If we assume that each row of $\mW^l$ and vector $\vb^l$ have about the same number of positive and negative elements, 
then we can expect that each set $I(l-1,j)$ contains half of the units in $U_l$. 
If these positive and negative elements are distributed randomly for each unit, 
then a logarithmic number of the units in layer $l-1$ being active may suffice 
to cover $U_l$. In such case, we would explore $O(n_{l-1})$ subsets on average.

\section{Approximation Algorithms for Computing $A_l(k)$ and $I_l(k)$}
\label{ap:submodular_bounds}
In Section~\ref{sec:ub}, the assume these bounds for $\mathcal{I}_l(k)$ and $\mathcal{A}_l(k)$:
\begin{equation*}
\begin{scriptsize}
\mathcal{I}_l(k) \leq
\max\limits_{S} \left\{ \left| \bigcup\limits_{j \in S} I(l-1,j)  \right| : S \subseteq \{ 1, \ldots, n_{l-1} \}, |S| \leq k \right\}
\end{scriptsize}
\end{equation*}

\begin{align*}
\mathcal{A}_l(k) \leq & n^+_l + |U^+_l| + \\ & \max\limits_{S} \left\{ \left| \bigcup\limits_{j \in S} I^-(l-1,j)  \right| : S \subseteq \{ 1, \ldots, n_{l-1} \}, |S| \leq k \right\}
\end{align*}

The maximization terms on the right hand side of the inequalities for $\mathcal{I}_l(k)$ and $\mathcal{A}_l(k)$ can be seen as finding a set of $k$ subsets of the form $I(l-1,j)$ or $I^{-}(l-1,j)$, respectively, and whose union achieves the largest cardinality. This can be shown to be directly related to the maximum k-coverage problem with $(1-\frac{1}{e})-$approximation using an efficient greedy algorithm~\cite{Feige1998}. Note that the maximum k-coverage problem is actually a special case of the maximization of submodular functions, which are discrete analogue of convex functions~\cite{George_Nemhauser_MP1978}. For large networks, the use of greedy algorithms can be beneficial to get good approximations for $\mathcal{I}_l(k)$ and $\mathcal{A}_l(k)$ efficiently.

\section{Additional Information on Networks and Experiments}
\label{ap:exps}

To compare with exact counting, the networks used in the experiments are those originally described in~\cite{BoundingCounting}.  
They consist of rectifier networks trained on the MNIST benchmakrk dataset~\cite{MNIST}, 
each with two hidden layer summing to 22 units and an output layer having another 10 units. 
For each possible configurations of units in the hidden layers, 
there are 10 distinct networks that were trained for 20 epochs. 
The linear regions considered are in the $[0,1]^{n_0}$ box of the input $x$. 
In the cases where no layer has 3 or less units, the number of linear regions was shown to relate to the network accuracy.

In our experiments, the code is written in C++ (gcc 4.8.4) using CPLEX Studio 12.8 as a solver and ran in Ubuntu 14.04.4 on a machine with 40 Intel(R) Xeon(R) CPU E5-2640 v4 @ 2.40GHz processors and 132 GB of RAM.

\paragraph{Acknowledgements} We would like to thank  Christian Tjandraatmadja, Michele Lombardi, and Tiago F. Tavares for their helpful suggestions and comments on this work.

\bibliographystyle{aaai}
\bibliography{empirical_bounds}

\end{document}